\documentclass{article}
\usepackage{graphicx} % Required for inserting images
\usepackage{amsmath, amssymb}
\usepackage{booktabs}
\usepackage[numbers,sort&compress]{natbib}
\usepackage{soul} % Strikeout Aid allocation table
\usepackage{ifthen} % used for revision commands
\usepackage[normalem]{ulem} % used for revision commands
\usepackage{amsmath}
\usepackage{color}
\usepackage{amsthm}
\usepackage{subcaption}
\usepackage{makecell}
\usepackage{algorithm}
\usepackage{algpseudocode}
\usepackage{url}
\usepackage[colorlinks=true, linkcolor=blue, citecolor=blue, urlcolor=blue]{hyperref}
\usepackage{multirow}

\theoremstyle{plain}
\newtheorem{theorem}{Theorem}

\NewDocumentCommand{\genericEdit}{+O{}+mm}{%Note: nested optional arguments need to be enclosed in {} -- e.g. \jb[ \jb[{this should be enclosed}]{}]{}.
    {\scriptsize\bf \color{#3}%
        \IfBlankTF{#1}{%
            #2%
        }%
        {%
            \IfBlankTF{#2}{%
                \stkout{#1}%
            }%
            {%
                \sout{#1}\ #2
            }%
        }%
        \normalsize%
    }%
}
\NewDocumentCommand{\jb}{+O{}+m}{\genericEdit[#1]{James: #2}{red}}
\NewDocumentCommand{\mbp}{+O{}+m}{\genericEdit[#1]{Markus: #2}{blue}}
\NewDocumentCommand{\ad}{+O{}+m}{\genericEdit[#1]{Adel: #2}{green}}
\definecolor{fableorange}{rgb}{0.90,0.45,0.00}
\NewDocumentCommand{\fable}{+O{}+m}{\genericEdit[#1]{Fable: #2}{fableorange}}
\definecolor{codexviolet}{rgb}{0.50,0.00,0.80}
\NewDocumentCommand{\codex}{+O{}+m}{\genericEdit[#1]{Codex: #2}{codexviolet}}
\RenewDocumentCommand{\jb}{+O{}+m}{}
\RenewDocumentCommand{\mbp}{+O{}+m}{}
\RenewDocumentCommand{\ad}{+O{}+m}{}
\RenewDocumentCommand{\fable}{+O{}+m}{}
\RenewDocumentCommand{\codex}{+O{}+m}{}

\NewDocumentCommand{\jbl}{+O{}+m}{\genericEdit[#1]{James (leave till after initial submission): #2}{red}}
\RenewDocumentCommand{\jbl}{+O{}+m}{}

\newcommand{\stkout}[1]{\ifmmode\text{\st{\ensuremath{#1}}}\else\sout{#1}\fi}

\newcommand{\EoMlModel}{EO-ML model}
\newcommand{\ExclusionRate}{exclusion rate among eligible recipients }

\newcommand{\tone}{\hat t_{\alpha_1}}
\newcommand{\ttwo}{\hat t_{1-\alpha_2}}

\DeclareMathOperator*{\Thresh}{\textnormal{Th}_{\alpha_1,\alpha_2}}
\DeclareMathOperator*{\Threshstar}{\textnormal{Th}_{\alpha_1,\alpha^*_2}}
\DeclareMathOperator*{\Safe}{\textsc{safe}}

\title{Beyond Point Predictions: Uncertainty-Aware Satellite Poverty Mapping for Public Policy}
\author{
    Markus B.\ Pettersson$^{1,2,3,*}$ \and
    James Bailie$^{1,3}$ \and
    Mohammad Kakooei$^{3,4}$ \and
    Eagon Meng$^{3,5,6}$ \and
    Adel Daoud$^{1,2,3,*}$
}
\date{
    {\small $^{1}$Division of Data Science and AI, Department of Computer Science and Engineering, Chalmers University of Technology and the University of Gothenburg, Gothenburg, Sweden} \\
    {\small $^{2}$Institute for Analytical Sociology, Link\"oping University, Norrk\"oping, Sweden} \\
    {\small $^{3}$AI \& Global Development Lab, Link\"oping University, Norrk\"oping, Sweden} \\
    {\small $^{4}$Geomatics, Department of Environmental and Life Sciences, Karlstad University, Karlstad, Sweden} \\
    {\small $^{5}$Department of Electrical Engineering and Computer Science, Massachusetts Institute of Technology, Cambridge, Massachusetts, USA} \\
    {\small $^{6}$Computer Science and Artificial Intelligence Laboratory, Massachusetts Institute of Technology, Cambridge, Massachusetts, USA} \\
    {\small $^{*}$Corresponding authors: Markus B.\ Pettersson (markus.pettersson@chalmers.se) and Adel Daoud (adel.daoud@liu.se)} \\[6pt]
    {\small \today}
}

\begin{document}

\maketitle

\begin{abstract} %\bfseries \boldmath

\noindent Despite their critical importance for policy and research, high-resolution poverty data remain limited across much of Africa. 
%Although this data can be supplemented by 
Machine learning (ML) with earth observation (EO) imagery has recently emerged as a way to supplement these data by predicting (i.e., estimating) poverty where it has not been directly measured.
%Estimates derived from earth observation (EO) imagery using machine learning (ML) have recently emerged as a supplement to these data, providing granular, continent-wide poverty maps for the first time. 
%While this provides granular, continent-wide poverty maps for the first time,
%While granular, continent-wide maps of these predictions can be produced, 
%Yet to be useful, users need assurances of how potential errors in these predictions could affect their results and decisions.
%Yet for these predictions to have value, users need assurances that their decisions are not accidentally the result of errors in these predictions.
%Yet for these predictions to be used reliably, users need assurances that they will not be misled by the errors in these predictions. 
%Yet to be used reliable, these predictions need to be accompanied with guarantees that users will not be misled by the errors in these predictions.
Yet to be used reliably, decision-makers and analysts need assurances that they will not be misled by the errors in these predictions. 
%potential errors in these will not affect their 
%how potential errors will impact their decisions.
%Yet when predictions, affect high-stakes decisions, 
%It is tempting to use these maps for policy, yet when...\jb{need a linking sentence here} While these estimates/maps explain a large fraction of variance of poverty have demonstrated remarkable performance (accuracy is not quite right, `explanatory power' is the best?, maybe `strong model fit'), Yet when predictions \jb{note use of estimates in previous sentence} affect high-stakes decisions, point accuracy alone is insufficient: decision-makers also need reliable uncertainty estimates \jb{not uncertainty estimates (this is not what we are providing with SAFE), but just something about reliability/guarantees.}. 
To meet this need, we develop an uncertainty-aware EO-ML method for poverty mapping based on simultaneous quantile regression and a novel form of conformal prediction. Using a spatiotemporal transformer trained on sequences of Landsat and nighttime-light images, we produce prediction intervals for neighborhood-level International Wealth Index estimates across Africa which are statistically guaranteed to achieve their desired coverage rates. While our method's point-prediction performance matches the state of the art, its prediction intervals are wider than might be expected given its high $R^2$ of $0.75$.
%we show that local uncertainty is much larger than aggregate metrics such as $R^2$ might suggest, yielding wide prediction intervals; because interval width is determined by the residual error of the underlying predictions, other EO-ML models of similar accuracy would likely face comparable uncertainty, suggesting an inherent limitation in existing methods, and corroborating with other literature that existing model's point predictions cannot always be relied upon for high-stakes decisions, despite their remarkable performance/explanatory power. 
However, other models %\jb{EO-ML is implicit, we don't need to repeat this again and use this term twice in a single sentence.}
of similar accuracy likely 
%because interval widths are fundamentally determined by the error of the underlying predictions, other models of similar accuracy %\jb{all other models have similar or worse accuracy, so we can drop this} 
suffer from comparable uncertainty, pointing to an inherent limitation: %in existing methods 
%and corroborating recent findings that, 
Even with its remarkably high explanatory power, EO-ML cannot naively be relied upon 
%for policy-making (what it was to begin with)
for policy-making, such as %/like/including 
when designing poverty-targeting programs.
%to inform policy like poverty-targeting programs.
%when designing public policy, including poverty targeting
%when making policy, such as poverty-targeting programs. 
%for policy-making, like poverty-targeting (Adel's initial edit)
To handle this challenge, we develop a procedure to efficiently allocate aid 
using both ground-truth surveys and model predictions while provably ensuring the risk of excluding eligible neighborhoods remains below a prespecified level. In simulations, this approach delivers substantially more aid per eligible recipient than other strategies,
%which delivers substantially more aid per eligible recipient compared to survey-only strategies while 
thereby demonstrating that EO-ML can indeed be a reliable supplement to traditional data sources---as long as methods are tailored to the problem at hand. %Under budget constraint, a policymaker can thus give more aid to neighborhoods in need rather than using their budget for surveying.\jb{I feel like this is repetitive, we just said ``this approach deliversy substantially more aid per eligible recipient'', now we are essentially saying that again.}
%\jb{This is already ~270 words, which is 20 words too long. But we also said we wanted to talk about ``Blumenstock says either survey or ML. We show how both of these can be combined in a disciplined way''.}
Taken together, this work establishes how survey estimates and EO-ML predictions can be combined to achieve efficiency gains beyond what is possible with either data source alone, 
%in a way that does not compromise 
%and 
without compromising 
the reliability of the resulting decisions.
% performs better than while also guaranteeing
% is guaranteed to reliable 
% , which has statistical guarantees 

% reducing unnecessary surveys and improving aid targeting.
% This demonstrates that, when targeting a specific use case, 
% Thus, 

% Nevertheless, by tailoring
% and combining with traditional 
% can still be leveraged to reliable supplement traditional surveying, 
% improving on survey-based methods
% achieving best of both worlds/higher efficiency than either on their own.

% we demonstrate that reliable EO-ML poverty mapping can still remain useful nevertheless remains useful/reliable in applications where uncertainty is directly incorporated into the decision rule
% when methods are tailored directly to the decision in question. In an aid-allocation application, our conformalized thresholding procedure caps the rate at which eligible communities are wrongly excluded while outperforming standard benchmarks, reducing unnecessary surveys and improving the allocation of transfers. Rather than requiring a choice between costly surveys and uncertain predictions, this framework combines them: survey resources are spent only where predictions are too uncertain to act on.

\end{abstract}

\section*{Significance Statement}

%Satellite imagery combined with machine learning now produces detailed poverty maps for regions where household surveys are scarce, and such maps increasingly inform who receives assistance \mbp{Add reference. Does AD have an example?} \jb{what about leaving this more general: e.g., ``increasingly inform policy and research'' instead of ``increasingly inform who receives assistance''. I don't see why we should just focus on ``who receives assistance''; I think ``who receives assistance'' is a little unclear (who out of what group? what assistance?); and making it more general will let us add all the standard references, e.g., Yeh et al., or the German group.}. We show that these maps, although accurate on average, carry local uncertainty large enough to change high-stakes decisions about individual communities. We develop methods that convert this uncertainty into explicit statistical guarantees: they certify when a neighborhood can be confidently classified as poor, abstain when the evidence is insufficient, and direct follow-up surveys only where they are needed \mbp{Abstain/follow-up survey is the same in SAFE}. In simulations of budget-constrained aid allocation calibrated on African survey data, uncertainty-aware targeting delivers substantially more aid per eligible recipient than survey-only strategies while capping the risk of excluding poor communities. Reliable uncertainty quantification thus turns a promising prediction technology into accountable policy infrastructure.
By analyzing satellite images, machine learning models are now producing detailed poverty maps for regions where household surveys are scarce. 
%Detailed poverty maps for regions where household surveys are scarce are now being created from satellite images by machine learning. 
%Satellite imagery, combined with machine learning, now produces detailed poverty maps for regions where household surveys are scarce, and such maps increasingly inform policy and research. 
%We show that, while 
To a large degree, these models explain the differences %variation \jb{maybe variation/variance is too technical}
in poverty levels across Africa, but
there is still substantial error in many of their estimates. 
%any one of their estimates can still be very uncertain.
%their estimates remain uncertain.
%This makes decision-making based on these maps 
Thus, while the poverty maps they produce are informative, using these maps without accounting for the uncertainty introduced by the models' errors can lead to incorrect decisions and unreliable findings.  
Because machine learning poverty estimates are increasingly informing policy and research, 
it is critical to have uncertainty-aware methods that provide explicit guarantees of the reliability of downstream analysis and decision-making.
%it is critical to have methods for incorporating 
%maps' uncertainty in downstream analysis and decision-making
%and providing explicit guarantees of their reliability
This article presents some of the first such methods tailored to poverty maps produced by machine learning with satellite images---a task which has its own unique challenges.

\section*{Introduction}

% \mbp{This became sort of a half-intro, half-abstract, but I think it conveys the framing.}

%\ad{James, can we make the story more cohesive and better integrated. The survey vs ML trade off is great way of working with this.}
Reliable and fine-grained information on poverty is a prerequisite for targeting scarce development resources, evaluating aid programs, and monitoring progress toward policy goals. Yet in much of Africa, data on living standards are drawn primarily from household surveys, which are costly to field, slow to update, and geographically incomplete \cite{groves_total_2010,lavrakas_encyclopedia_2008,dhs}. %Decision-makers and researchers have not had access to the data they need, making it 
%Thus, it remains 
%hard to formulate public policy interventions---such as deciding where to build the next road, school, or hospital---that benefit the neighborhoods most in need\jb{@Adel, is there a citation we can add for this?}. %JB: Note Adel your initial writing said: to decide where to build the next X, Y, or similar public policy interventions. This reads as: you want to build the next public policy intervention.
%Adel's old writing: thus, it remains hard for policymakers to formulate decisions on where to build the next road, school, hospital, or similar public policy intervention to benefit neighborhoods in need \cite{groves_total_2010,lavrakas_encyclopedia_2008,dhs}. 
%This creates a challenge:
%\jb{this sounds like an LLM -- is mismatch the right word? Is `persistent' the right adjective?}
Thus, the decisions that most benefit from high-resolution, up-to-date poverty data are frequently made in precisely the settings where such data are hardest to come by. %JB: reverting to this old sentence, because otherwise the flow of the paragraph doesn't fit with the next paragraph (e.g., the "complementary path")

%\ad{Markus point: Blumenstock says either survey or ML. We show how both of these can be used in a disciplined way. Now grounded as the framing for the paper.}

%To alleviate the scarcity of data on living conditions, r
Recent work suggests a complementary path that draws on the growing archive of earth observation (EO) imagery. By training machine learning (ML) models to predict asset wealth and related proxies from satellite data, EO-ML systems can generate neighborhood-level poverty estimates over large areas using limited ground-truth supervision \cite{Jean2016,Yeh2020,chi_microestimates_2022,pettersson2023,Wang2024}. In many settings, these approaches achieve strong performance, with reported $R^2$ values between $0.56$ and $0.76$, raising the prospect of high-resolution poverty mapping at scale \cite{pettersson2023}.
%In many settings, these approaches achieve strong performance, with reported $R^2$ values exceeding $0.75$ \jb{Barely... And I'm not sure we want to emphasise this, because then it makes our model (which has $R^2 = 0.75$) look bad}, raising the prospect of high-resolution poverty mapping at scale \cite{pettersson2023}.

However, strong average performance does not ensure that EO-ML estimates are useful for decision-making. 
%An $R^2$ summarizes average fit, but it does not reveal how errors vary across space or for the specific locations where decisions are made. 
%\jb{This needs to be changed -- $R^2$ is relative accuracy metric} \ad{I suggest for this paper to leave this R2 as it is.} 
$R^2$ is the fraction of variation observed in the data that is explained by a model. It is a measure of relative accuracy---relative to the baseline of predicting the average---rather than a measure of absolute accuracy. % it does not reveal how accurate 
Thus, a model can achieve $R^2>0.75$ while still %making large, systematic errors for certain regions, or while 
producing residuals that are too large for threshold-based aid targeting or other policy-making.
%A model can achieve $R^2>0.75$ while still making large, systematic errors for certain regions, or while producing wide \jb{wide doesn't make sense here} residuals that are unacceptable for threshold-based targeting. 
Even if EO-ML point predictions were more accurate than what is currently possible, they would still not be sufficient for decision-making and research. This is because 
even small errors in a model's predictions can matter: if a village just below the poverty line is predicted to be above the line, they will miss out on receiving aid from a program that uses EO-ML poverty estimates to assess eligibility. % This matters most near policy cutoffs: two villages with identical point predictions can imply very different risks of being below the poverty line depending on their prediction intervals, which changes whether one should survey, target, or defer. 
%Prediction intervals account for model errors %make this decision-relevant uncertainty explicit 
%by quantifying how wrong a prediction could plausibly be, rather than assuming that a point estimate is sufficient. 
%Uncertainty quantification 
%Addressing model error is therefore essential for turning EO-ML predictions into evidence that can be used responsibly. 
%EO-ML point predictions on their own are not a sufficient basis for policy and research. 
%What is required, before this new source of data can be used reliably, 
Thus, before this new source of data can be used reliably, it is essential to address the error introduced by EO-ML 
by quantifying the uncertainty in its predictions---that is, by quantifying how wrong each prediction could plausibly be.
%Addressing model error is therefore essential for ensuring that EO-ML predictions can be used as reliable evidence. 
%This requires uncertainty quantification---that is, quantification of how wrong each prediction could plausibly be.
%More specifically, 
%To do this, we need to quantify how 
%To do this, we need uncertainty quantification: 
%To do this, What is required is to quantify how wrong each prediction could plausibly be. 
%uncertainty quantification
%Quantifying how wrong a model could plausibly be is therefore essential for turning EO-ML predictions into evidence that can be used responsibly in such a sensitive domain.

%\jb{Define the term EO-SQR model somewhere in the introduction}
In this work, we develop an EO-ML approach that yields statistically valid uncertainty quantification. First, we train a model with simultaneous quantile regression \citep{sqr2018} to produce prediction intervals for an asset-based proxy of poverty. Using conformal prediction \citep{vovk2005, cqr_romano}, we then transform each prediction interval so that it is guaranteed to cover the true value of the proxy at the prespecified coverage rate under mild assumptions. %\jb{yielding twice. Maybe split this sentence and explain what coverage guarantees mean in more detail and how that connects to reliability} intervals with coverage guarantees under mild assumptions.
We show that, despite matching state-of-the-art EO-ML poverty models in terms of $R^2$ and accuracy, the resulting prediction intervals are typically wide, with the median interval covering two fifths of the observed values. 
These large interval widths are ultimately driven by the residual errors in the underlying point predictions: the larger those errors, the wider the intervals must be. 
In general, prediction intervals must have width inversely related to the model's accuracy, regardless of the model in question.
%regardless of which EO-ML model is used, any calibrated prediction intervals must have width matching 
%inversely related to the EO-ML model's performance
%the residuals of the 
%Any calibrated interval procedure must expand wherever residuals are large or heterogeneous, and interval width is ultimately driven by the residual error in the underlying point predictions: the larger that error, the wider the intervals needed to maintain coverage. 
Thus, because our model's accuracy is on par with prior EO-ML results \cite{pettersson2023, marty2024global, Wang2024, zheng2025dynamic, Yeh2020, sherman2026global, Jean2016, chi_microestimates_2022}, similarly wide intervals would arise with those models too. This finding suggests that previous work, which focused on point metrics and $R^2$ performance, may understate the uncertainty remaining in their predictions, and it clarifies when and where EO-derived estimates can be trusted.

%Wide intervals do not imply that EO-ML poverty mapping is useless \jb{I think this framing is not quite right. In some sense they are useless, although I don't think we want to say that. Rather, wide intervals mean that we need to be smarter. For example, we need to tailor our UQ directly to the decision/policy questions we are interested in. We show that this is possible with an aid-allocation example.}; rather, they motivate using these predictions in decision frameworks that explicitly account for uncertainty. 

The wide intervals we observed also motivate a tailored approach to designing EO-ML systems. 
%In particular, this finding motivates a tailored approach to designing EO-ML system.
Clearly, in order to build an EO-ML system which is reliable---i.e., which provides uncertainty guarantees---while still being efficient---e.g., not having wide prediction intervals---we must make the most out of the information available by tailoring to the downstream use in question.
To provide a concrete example of this approach, we consider the problem of determining which areas fall below a given poverty line. Using a novel form of conformal prediction, we develop a method that controls the rate of misclassification according to the risk tolerance of the policy-maker. 
We then extend this method to address the question of how to allocate aid to communities that need it the most.
%\jb{In particular, to build reliable and efficient EO-ML systems, they must be tailored to the downstream use in question. decision, in order to get the most out of the power. (Don't optimize generally, do something specific/tailored)}
%Wide intervals mean that uncertainty quantification must be tailored to the decision it serves\jb{thereby getting the most out of the data tailored to the particular question in mind}; they motivate using these predictions in decision frameworks that explicitly account for uncertainty. 
%To illustrate this, \jb{talk about conformalized thresholding first}we present an aid-allocation example\jb{not example---it's a new method} that targets villages below a poverty threshold while providing statistical guarantees on the false negative rate (FNR)\jb{Edit out mention FNR--on the rate of misclassifying such villages}. 
In this setting, by spending survey resources only where the EO-ML model is genuinely uncertain, our method can support reliable targeting of aid at substantially lower cost than traditional survey-only approaches. At the same time, it guarantees that the fraction of eligible neighborhoods not receiving aid is capped below a preset threshold. %making the remaining uncertainty transparent rather than hidden. 
Recent evidence suggests that the cost of identifying eligible recipients rivals the size of the aggregate poverty gap (the total amount of money needed to lift everyone above the poverty line), 
%costs of actually lifting them out of poverty, 
strengthening the case for targeting procedures like ours that leverage EO-ML to reduce such costs
%spend survey resources only where the model is genuinely uncertain 
\citep{sahoo2025would}.

Overall, this work addresses what has become an either-or view of poverty measurement: either field costly surveys to directly collect data, or accept potentially unreliable EO-ML model predictions \citep{blumenstock_estimating_2018, aikenWhenShouldBig2025}. 
Breaking this dichotomy, our results show that EO-ML methods can reliably and efficiently be used in conjunction with surveys. %using the predictions that are certain, and surveying when not.
%Aggregate metrics such as $R^2$ alone cannot settle this choice, because any individual prediction can still be far off. Our work shows how EO-ML model can be used in conjunction with surveying---but for sampling only neighborhoods with large uncertainty.%\jb{Note that ``sampling only neighborhoods with large uncertainty'' is exactly what the next sentence says, in clearer terms} 
%\jb{this is too strong}.
As every prediction carries a calibrated statement of its own uncertainty, predictions and surveys become complements: model output is acted on where it is decisive, and survey resources are reserved for the places where it is not.
%\ad{Markus point: Blumenstock says either survey or ML. We show how both of these can be used in a disciplined way. Now grounded in the literature.}

%\jb{One way to write out this argument: ``Despite their remarkably high accuracy on-aggregate, any individual prediction can still be wildly incorrect. This lack of reliability has led some researchers to conclude ML predictions are only useful in limited scenarios, with direct measurements required the rest of the time (cite Blumenstock). Our work demonstrates how to intelligently augment direct measurements with EO-ML predictions, leading to much better tradeoffs than current all-or-nothing approaches.'' How does that sound? Feel free to edit.}

\section{Methods}
\label{sec:method_overview}

%This section describes our approach for a general audience\jb{or ``our approach in general terms''}\jb{I'm not sure we want to say ``a general audience'' since the PNAS is still an academic journal; we are not writing for say The New York Times. On further reflection, I think this sentence is just redundant.}; all formal definitions, algorithms and proofs are given in Appendix~\ref{app:detailed_methods} \jb{The second half of this sentence has been moved to the end of Methods}.
For a more detailed description of our methods, including all formal definitions, algorithms and theory, see Appendix~\ref{app:detailed_methods}.

\subsection{Measuring poverty}

Our outcome is household asset wealth, as measured by the International Wealth Index (IWI), which ranges from 0 (owning none of the tracked assets) to 100 (owning all of the tracked assets, including a TV, car, flush toilet, and three or more rooms) \citep{Smits2015}. We source IWI data from the Demographic and Health Surveys (DHS), a survey programme funded by the United States Agency for International Development (USAID) that interviews households in small geographic clusters (roughly a village in rural areas or a neighborhood in urban ones) \citep{dhs, Burgert2013}. We average household IWI scores within each cluster and treat these neighborhoods as the units for which we make predictions and decisions. In total, our data cover roughly 70{,}000 neighborhoods across 38 African countries between approximately 1990 and 2020 \citep{dhs_harmonization}.

%Old measuring poverty section:
% \subsection{Measuring asset wealth}

% Our outcome of interest is asset wealth measured by the International Wealth Index (IWI), a survey-based welfare proxy that places households on a comparable scale from 0 to 100 \citep{Smits2015}. We use IWI derived from Demographic and Health Surveys (DHS), which provide nationally representative household microdata across many African countries and are widely used in poverty measurement research \cite{dhs}. Following standard practice, we aggregate household-level values to the DHS cluster level, roughly equal to a village in rural areas and a neighborhood in urban areas, and treat each cluster as the unit of prediction and decision-making \citep{Burgert2013}. From here on, we refer to these clusters as neighborhoods.

% In total, we use IWI data for 69,949 DHS neighborhoods (about 1.5 million households) from 38 African countries, spanning 164 DHS surveys fielded from the early 1990s through to the early 2020s \citep{dhs_harmonization}. For confidentiality, the DHS randomly displaces cluster coordinates (up to 2~km for urban clusters and 5~km for rural clusters, with 1\% of rural clusters displaced up to 10~km), a perturbation that our $6.72$~km input footprint of our satellite imagery largely absorbs \citep{Burgert2013}.

\subsection{Predicting wealth from satellite images}
For each neighborhood, we assemble a time series of multispectral Landsat \citep{landsat45, landsat7, landsat8} and nighttime-light \citep{Li2020} satellite images covering a $6.72 \times 6.72$~km footprint centered on its location. We train a transformer-based simultaneous-quantile-regression model, which we call the EO-SQR model, to learn the relationship between these images and the neighborhood's IWI. Unlike conventional regression models, which return a single best estimate, simultaneous quantile regression takes a given quantile level $q$ as input and predicts the value below which the true neighborhood-level IWI is expected to fall with probability $q$, conditional on the neighborhood's satellite imagery \citep{sqr2018}. We obtain a point estimate for each neighborhood by predicting the 50th percentile, which is equivalent to the prediction produced by a standard regression model. To construct a prediction interval with a target coverage level of $\alpha = 90\%$, we predict the 5th and 95th percentiles and use them as the lower and upper interval bounds, respectively (Figure~\ref{fig:ct_panel}A).

\subsection{Calibrating prediction intervals with conformal prediction}

The naive 90\% intervals produced by the transformer model are not calibrated, as the true IWI values are not guaranteed to lie in these intervals 90\% of the time. This can be amended using conformalized quantile regression (CQR) \citep{cqr_romano}. The idea is simple: set aside a labeled dataset that the EO-SQR model has never seen during training (a \emph{calibration set}), evaluate how much the prediction intervals miss the true labels on this data, sort the resulting residuals in ascending order, and select the residual corresponding to the desired coverage level. At deployment, this residual is then used as a uniform correction term, widening every prediction interval by the calibration-derived error margin to achieve the target coverage guarantee. Under a mild assumption (that calibration and new neighborhoods are exchangeable), the corrected intervals are mathematically guaranteed to contain the true wealth level at the target rate $\alpha$, no matter how complex the underlying model is. The guarantee is \emph{marginal}: it holds on average across neighborhoods, not for each named neighborhood individually.

\begin{figure}[!htbp]
    \centering
    \includegraphics[width=\linewidth]{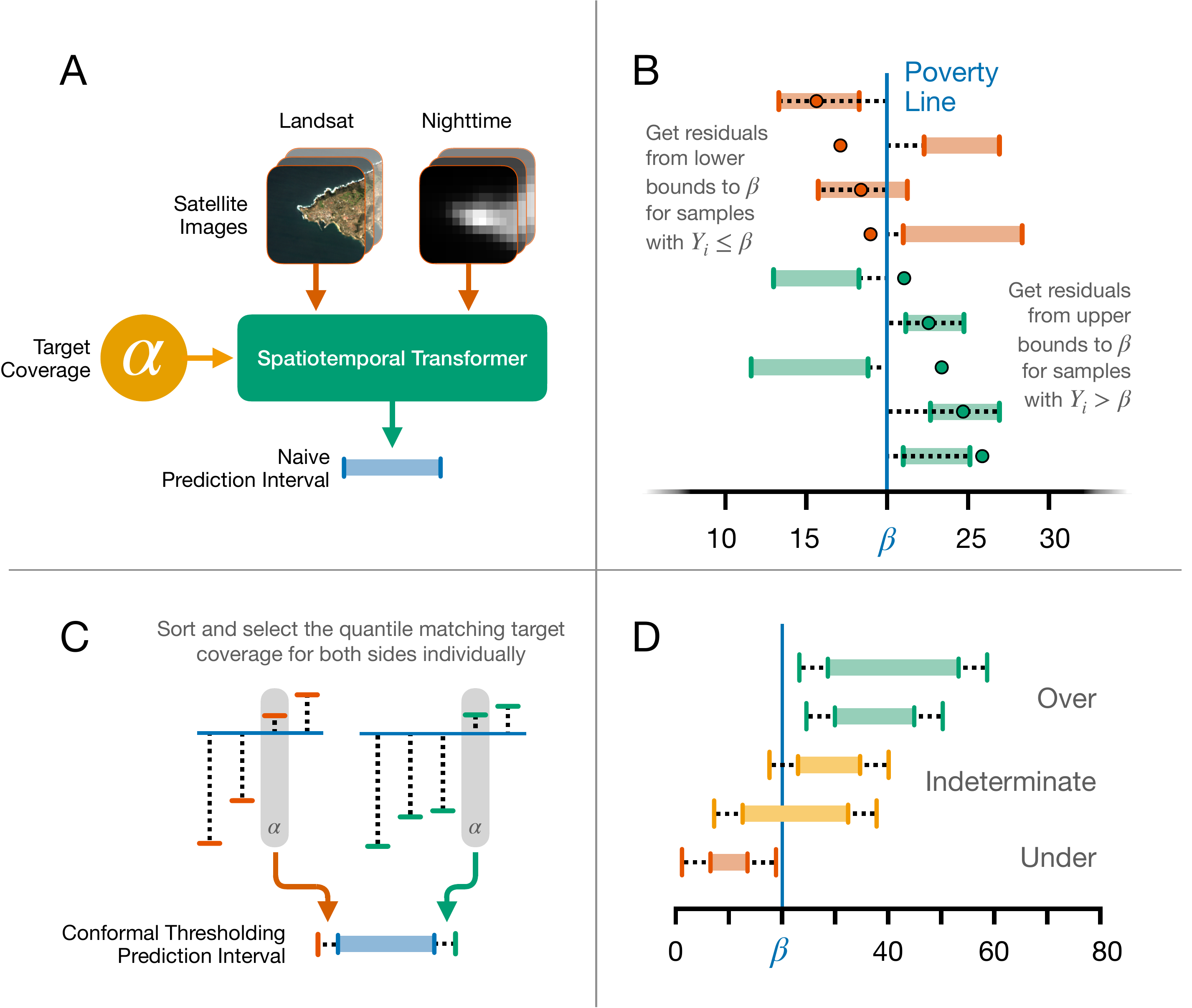}
    \caption{\textbf{Pipeline for Conformalized Thresholding (CT)}. \jbl{Maybe remove the 80 in the x-axis of (D) and just have the x-axis fade off on the right like in (B)} \textbf{(A)} The simultaneous quantile-regression model maps Landsat and nighttime-light imagery to an initial prediction interval. These model-based intervals adapt to the input location but have no finite-sample coverage guarantee before calibration. \textbf{(B)} Using a held-out, labeled calibration set, CT measures one-sided errors relative to the policy threshold $\beta$ separately for neighborhoods with $Y_i\leq\beta$ and for neighborhoods with $Y_i>\beta$. \textbf{(C)} From each set of errors, CT selects the conformal quantile associated with the target coverage rate $\alpha$ and uses it to adjust the corresponding interval bound, producing a calibrated thresholding interval for a new neighborhood. \textbf{(D)} If the calibrated interval lies entirely above $\beta$, the neighborhood is classified as ``Above''; if it lies entirely below $\beta$, it is classified as ``Below''; and if it overlaps $\beta$, CT abstains and returns ``Indeterminate.'' Under exchangeability, both the probability of misclassifying a neighborhood with $Y\leq\beta$ as ``Above'' and the probability of misclassifying a neighborhood with $Y>\beta$ as ``Below'' are at most $1-\alpha$.}
    \label{fig:ct_panel}
\end{figure}

\subsection{Conformalized thresholding for poverty classification}

%\begin{figure}[!htbp]
%    \centering
%    \includegraphics[width=0.5\linewidth]{images/thresholding_task.png}
%    \caption{Prediction intervals turn threshold classification into a three-way decision. Each bar is one neighborhood's calibrated wealth interval on the 0--100 IWI scale, and the dashed line is the poverty threshold. If the whole interval lies above the threshold (green), the neighborhood is confidently classified as ``Over''; if the whole interval lies below (orange), as ``Under''; if the interval straddles the threshold (yellow), the data do not support either call and the method abstains, returning ``Indeterminate.'' \jb{I suggest combining Fig 1--3 into a single figure, so then we have a single ``methods'' figure. It would be great if we had a single figure such that a reader could look just at this figure (having only read the abstract) and understand our method. At the very least, we should combine Fig 1 and 2, since they are both narrow.} \mbp{Call it $\beta$}}
%    \label{fig:thresholding_task}
%\end{figure}
%\ad{James, change title}

Many policy questions reduce to a threshold: is this neighborhood below the poverty line $\beta$ or not? Prediction intervals give a natural decision rule, illustrated in Figure~\ref{fig:ct_panel}D: classify a neighborhood only when its entire interval falls on one side of the threshold, and otherwise abstain and say ``indeterminate.'' Our contribution, Conformalized Thresholding (CT), calibrates these prediction intervals so that both types of errors are controlled at the pre-set level $1-\alpha$: 
%\jb{Change this to the same $\alpha$, then say in the next sentence (or something) that we can have two different $\alpha$'s, so that we can have varying the error rates for each group, because making an error for the poor can matter more/less than making an error for the rich} 
among neighborhoods that are truly below the poverty line, at most a fraction $1-\alpha$ are wrongly cleared, and among neighborhoods above the line, at most a fraction $1-\alpha$ are wrongly flagged. Moreover, because one type of error can be more costly than the other, our method allows a practitioner to set the two error rate levels separately, ensuring an $\alpha_1$-coverage guarantee for neighborhoods below the poverty line, and an $\alpha_2$-coverage guarantee for those above. We obtain these two guarantees by calibrating the two sides of the decision separately on the corresponding groups of calibration neighborhoods, similar to class-conditional conformal prediction \citep{angelopoulos2023conformal}. %\jb{Not exactly the same as class-conditional in Angelopoulos and Bates, since we have a different objective}
Subject to these error caps, a good method should abstain as rarely as possible. Figure~\ref{fig:ct_panel} walks through the calibration mechanics: residuals are collected separately for the two groups of calibration neighborhoods, a group-specific quantile sets each correction, and new intervals are padded so that threshold decisions err on the safe side at the chosen rates.

\subsection{From classification to aid allocation: SAFE}

To demonstrate the practical utility of these methods for policymaking, we simulate an aid-allocation problem. Given a fixed budget and an eligibility threshold $\beta$, we seek to determine which neighborhoods in a country are eligible for cash transfers. We have access to our EO-SQR model, a small survey that serves as a calibration set, and the option to survey any neighborhood at a predetermined cost to establish its true eligibility. The aim is to maximize the fraction of the budget spent on cash transfers to eligible neighborhoods while respecting a predetermined upper bound $1-\alpha_1$ on the \ExclusionRate.

To this end, we extend CT into a procedure called Screen And Follow up with Error control (SAFE). SAFE first screens out neighborhoods that are confidently above the eligibility threshold $\beta$, with the risk of wrongly excluding a neighborhood capped at the predetermined exclusion rate $1-\alpha_1$, as in CT. Among the neighborhoods that remain, those whose CT prediction intervals lie entirely below $\beta$ receive aid directly, whereas indeterminate neighborhoods are surveyed. This creates a trade-off: providing aid directly risks allocating funds to ineligible neighborhoods, whereas surveying diverts resources that could otherwise fund transfers. The coverage level $\alpha_2$ governs this trade-off. 

SAFE selects $\alpha_2$ value by iteratively evaluating all available values from 0, at which all remaining neighborhoods receive aid, to 1, at which all are surveyed. It then chooses the value that maximizes a transparent utility measure, aid dollars per truly eligible recipient, estimated from the calibration data.

\subsection{Evaluation design}

We evaluate our methods using five-fold cross-validation. In each cross-validation round, three folds are used to train the EO-SQR model, one fold is used to estimate the conformal corrections, and the remaining fold is used exclusively for testing. Conformal corrections are estimated separately in each cross-validation round using only its designated calibration fold and are then applied to that round's test fold. Unless otherwise stated, the folds are constructed using an out-of-area splitting procedure that spatially separates neighborhoods and prevents overlapping satellite footprints from appearing across partitions \citep{pettersson2023}. %Details of this procedure are provided in Appendix~\ref{sec:ooa_split}.

By default, we compute evaluation metrics from the pooled out-of-fold test predictions. For the SAFE evaluation, we instead partition the data by country to represent deployment in countries excluded from model training. The countries assigned to the training, calibration, and evaluation partitions for each fold are also listed in Supplementary Table~\ref{tab:fold_countries} (in Appendix~\ref{sec:ooa_split}).

\section{Results}

\subsection{Conformalized quantile-regression estimates}

\begin{figure}[!htbp]
    \centering
    \includegraphics[width=\linewidth]{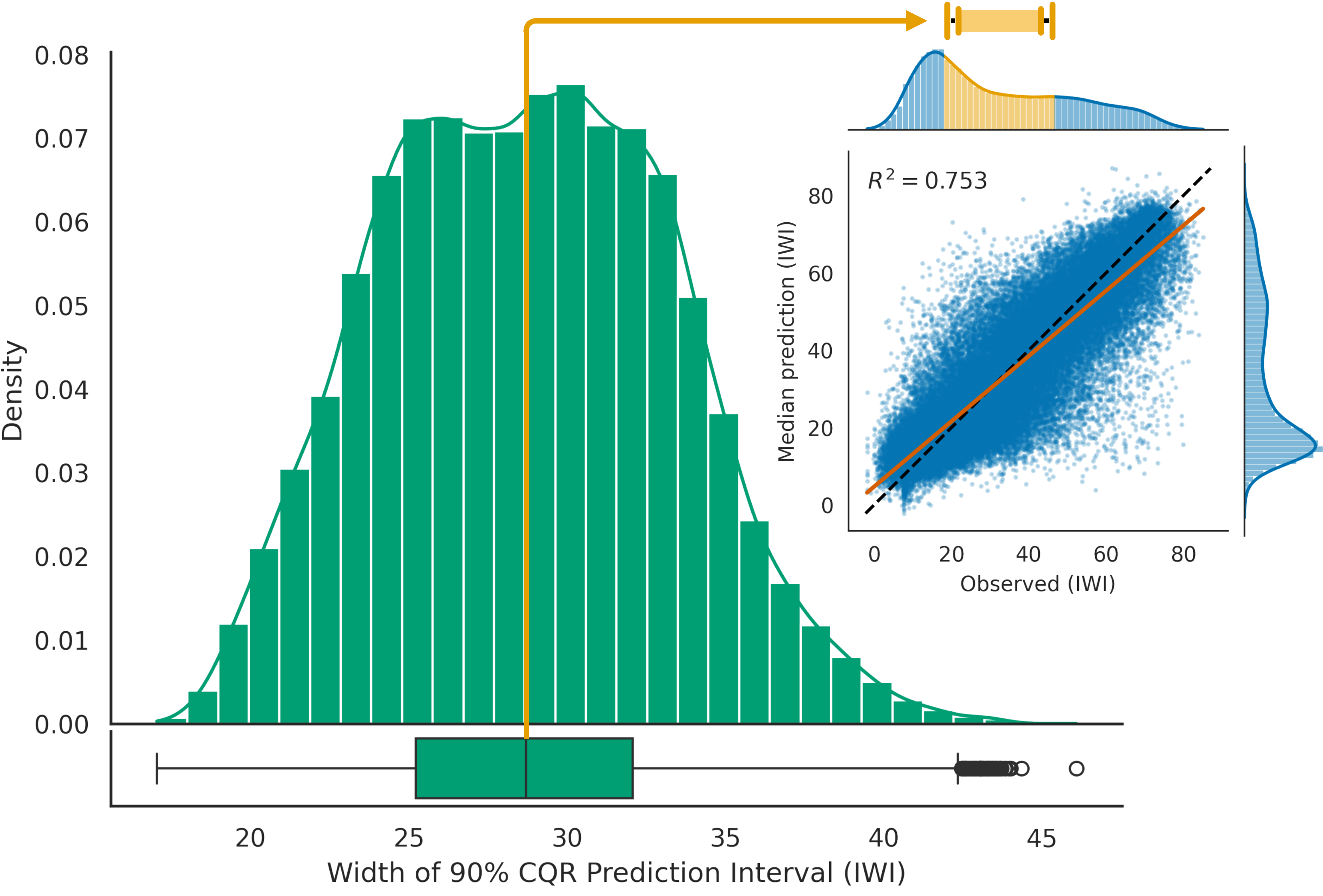}
    \caption{\textbf{Strong aggregate predictive performance coexists with substantial local uncertainty.} The main panel shows the distribution of 90\% CQR prediction-interval widths across held-out DHS neighborhoods. Although the inset shows strong aggregate agreement between observed IWI and the median predictions ($R^2=0.753$), matching reported values in prior works, the median interval is still 28.7 IWI points wide (orange). When centered on the mean observed IWI, an interval of this width contains 39.7\% of the observations in the full dataset, illustrating how limited the model's local precision remains despite its strong aggregate performance.}
    \label{fig:hist_and_cal}
\end{figure}

Much like prior EO-ML work on asset-wealth prediction, our EO-SQR model achieves strong point-prediction accuracy. Using the model's median predictions, we obtain $R^2=0.75$ on held-out survey locations, indicating that the model explains a substantial share of out-of-sample variation in asset wealth while still leaving meaningful residual error (Figure~\ref{fig:hist_and_cal}). Although direct comparisons with earlier studies are difficult because of differences in data, assumptions, and evaluation setups, this performance appears to be matching the state of the art in the existing EO-ML literature. Across eight recent papers on the topic, the reported $R^2$ values range from $0.56$ to $0.76$, placing our model near the top of the reported performance range (see Appendix~\ref{sec:previous_works} for further discussion). 
Any calibrated interval procedure must expand where residuals are large or heterogeneous, and prediction interval width is ultimately driven by the residual error of the underlying point predictions: the larger that error, the wider the intervals needed to maintain coverage. Because our model performance match prior EO-ML results, similarly wide intervals would likely arise for those models if they were required to achieve the same coverage.
%Any calibrated interval procedure must expand where residuals are large or heterogeneous, and \jb{prediction?} interval width is ultimately constrained \jb{Is constrained the right word here? Constrained sounds like the width is being limited/upper bounded by the error, but actually it is the opposite: the error makes the width larger.} by the irreducible error of the underlying point predictions. Because our point model matches prior EO-ML results, similarly wide intervals would likely arise for those models if they were required to achieve the same coverage. %\jb{If reviewer challenges us on this point, we can run conformal on another model and show that it gets wide intervals. We can still get conformal prediction intervals even if the model only produces point estimates, right?}

To quantify model uncertainty, we use conformalized quantile regression to construct 90\% prediction intervals for each neighborhood-level estimate and assess their performance on held-out DHS neighborhoods. %As shown in Figure~\ref{fig:qq_plot}, the \jb{moving this to the end of the sentence since this figure is in the appendix}
The naive intervals produced by our \EoMlModel already come close to the target coverage, but they still require conformal calibration to satisfy this requirement (see Supplementary Figure~\ref{fig:qq_plot}). Figure~\ref{fig:hist_and_cal} shows the distribution of interval widths. Although point predictions are accurate on average, interval widths remain large for most locations relative to the observed wealth range. To illustrate the practical magnitude of this uncertainty, the median interval width (28.7), when centered on the mean IWI score across all surveys, yields a span from 18.1 to 46.8, encompassing households with basic floor materials and few durable assets to households with televisions and refrigerators. As a result, even in the average case, these estimates are likely too uncertain for many operational tasks that require confident local-level decisions, such as threshold-based targeting or rank-based allocation. This highlights the limitations of relying on EO-ML estimates without careful consideration of uncertainty, while also motivating a broader view of what these models can still deliver at continental scale.

To demonstrate the scalability of this EO-ML approach, we created a 6.72 km/pixel raster covering all populated places on the African continent with predictions of IWI levels for the year 2021. The resulting gridded asset-wealth map and CQR prediction intervals (Figure~\ref{fig:iwi_and_ci_maps}) are made available through the Harvard Dataverse. More details on the map creation process are provided in Appendix~\ref{sec:map_creation}.
%\mbp{AD, could you clarify this?} \mbp{Mention that it doesn't simply correlate with wealth.}\fable{To Markus: rewritten as a descriptive association in response to your two queries. If you want the stronger statement (width tracks variability, not wealth), it needs a reported correlation or a stratified check against the full-resolution raster; Adel can add that once your regenerated maps land.}
%Interval width thus tracks outcome variability rather than survey coverage alone --- the model is most certain precisely where nearly everyone is poor, and least certain where wealth is mixed. \mbp{AD, could you clarify this?} \mbp{Mention that it doesn't simply correlate with wealth.}

%PAIP: Purpose, % Location and summary: "from this point of view", Highlighst: "when FPR increases, ...", Conclusion and implications: "Thus"

\begin{figure}[!htbp]
    \centering
    \includegraphics[width=1\linewidth]{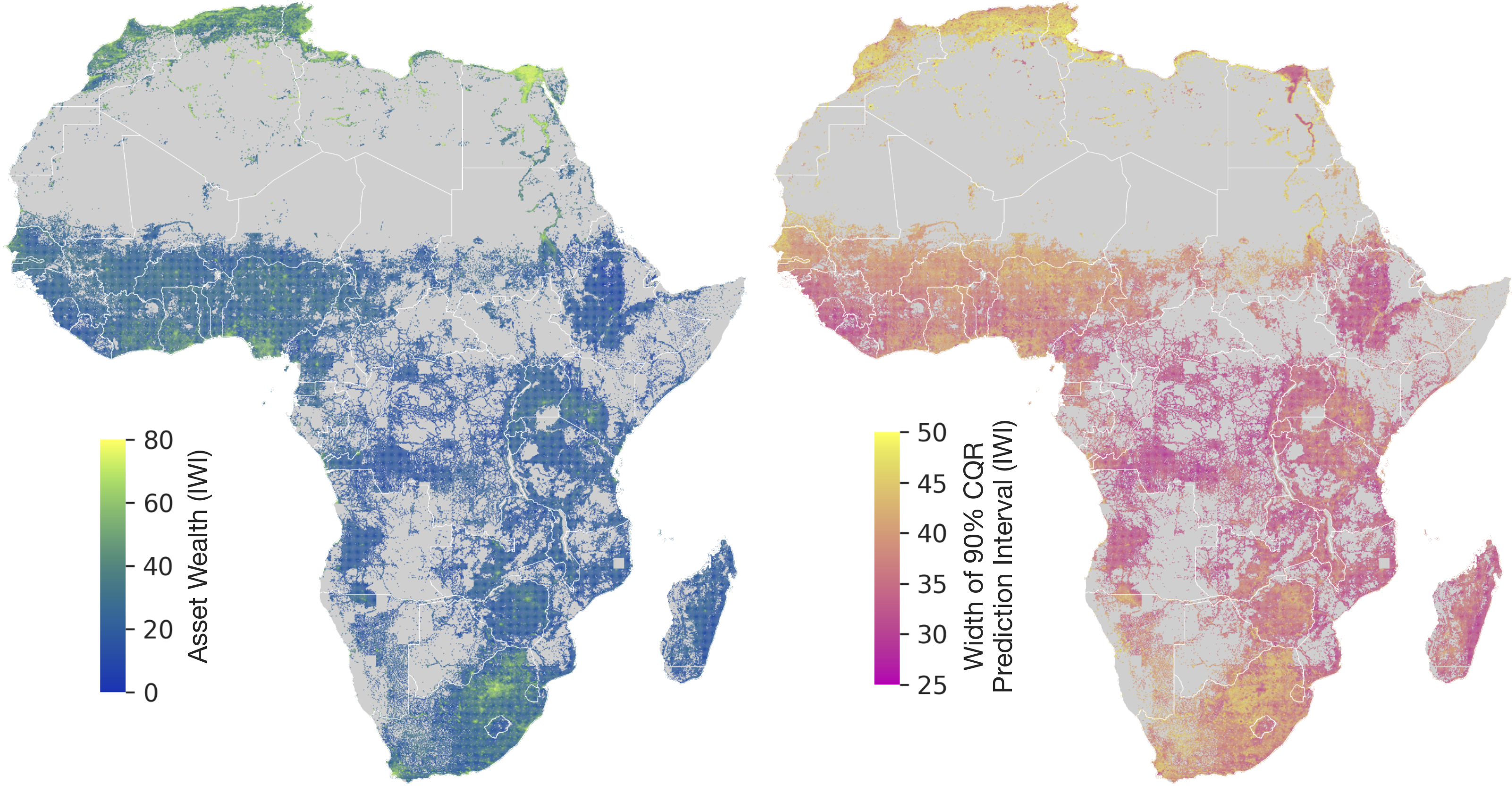}
    \caption{Neighborhood-level maps of predicted asset wealth (left) and uncertainty (right), measured as the width of the 90\% prediction interval given by conformal quantile regression. Predicted wealth is highest in North Africa, along the Nile valley, in South Africa, and around cities and transport corridors, while most of the sub-Saharan interior shows low predicted IWI. There are clear spatial patterns for uncertainty as well, but they are not clearly correlated with wealth. For example, most of the relatively wealthy North Africa have wide prediction intervals, while Egypt does not.}
    \label{fig:iwi_and_ci_maps}
\end{figure}

\subsection{Reliable EO-ML classification (CT)}
\label{sec:ct_res}

\begin{table}[!htbp]
    \centering
    \begin{tabular}{rlllll}
    \toprule
     & \makecell[l]{Error Rate \\ ``Below''} & \makecell[l]{Error Rate \\ ``Above''} & \makecell[l]{Decision \\ Coverage} & \makecell{Accuracy} & \makecell{Error Rate} \\
    \midrule
    Point prediction & \textcolor{red}{0.201} & \textcolor{red}{0.133} & \st{1.000} & \st{0.843} & \st{0.157} \\
    Naive intervals & 0.027 & 0.004 & 0.465 & 0.453 & 0.012 \\
    CQR intervals & 0.009 & 0.000 & 0.340 & 0.337 & 0.003 \\
    CT (ours) & \textbf{0.050} & \textbf{0.050} & \textbf{0.693} & \textbf{0.644} & \textbf{0.050} \\
    \bottomrule
    \end{tabular}
    \caption{Results on using the different methods for threshold classification when the threshold $\beta$ is 20 IWI and the target coverage $\alpha$ is 0.95. Our ``Conformalized thresholding'' (CT) method achieves the highest decision coverage of all methods respecting the $\alpha$ coverage. Red entries mark error rates that exceed the 5\% targets; struck-out values are displayed for completeness but are not valid comparators, since the method attains them only by violating the error constraints. We use the total number of samples, not just the number of predictions, in the denominator when calculating Accuracy and Error Rate.}
    \label{tab:james_table}
\end{table}

To illustrate the need for the extended conformal procedure, consider the task of determining whether a location falls below a fixed poverty-line, $\beta$, set to 20 IWI. The objective is to classify neighborhoods as ``Below'' or ``Above'' this threshold, while capping the error rate for both classes at 5\%. Subject to these constraints, we aim to maximize decision coverage, i.e., the share of neighborhoods for which the method makes a definitive classification rather than returning ``Indeterminate.'' As in Figure~\ref{fig:ct_panel}D, a neighborhood is classified as ``Below'' if its entire prediction interval falls below $\beta$ and as ``Above'' if the entire interval falls above.

As seen in Table~\ref{tab:james_table}, simply classifying all units based on their point prediction naturally leads to full decision coverage, but violates both error constraints ($0.20$ for below, $0.13$ for above). The naive intervals of our EO-SQR model satisfy both thresholds in this instance, despite not having a finite-sample guarantee, but classify fewer than half the samples. The traditional Conformal Quantile Regression intervals with $\alpha=0.95$ add statistical validity, but are overly conservative here, reducing decision coverage further to 34.0\%. By contrast, our Conformal Thresholding method meets both target error rates exactly ($0.05$ for below, $0.05$ for above) while retaining substantially higher decision coverage, classifying 69.3\% of samples.

\subsection{Reliable EO-ML for aid allocation (SAFE)}

We evaluate SAFE in simulations covering the seventeen countries whose DHS surveys were completed after 2020. To approximate deployment without using local observations to train the EO-SQR model, we generate predictions for each evaluation country using an EO-SQR model trained exclusively on surveys from the other countries in the dataset. We define eligibility using a country-specific poverty threshold, $\beta$, set to the first quartile of the evaluation country's surveyed IWI distribution. We assume a fixed survey cost of $c=100$ per neighborhood and a total budget of $T=cN$, where $N$ is the estimated number of neighborhoods in the country. As a DHS enumeration area may include up to 300 households, we approximate $N$ by dividing the country's population by 300 \citep{Elkasabi2020, tatem_worldpop_2017}.

For decision-makers, the central risk parameter is the maximum share of eligible neighborhoods they are willing to leave without aid. We call this the \ExclusionRate and vary it from 0 to 0.30 as a robustness check. A target of zero is the most conservative: no eligible neighborhood may be excluded. Meeting this target generally requires more surveys and/or allocating aid to more ineligible neighborhoods, leaving less of the budget available for eligible neighborhoods. Allowing for a higher exclusion rate relaxes this constraint and can increase the amount available for each recipient. We therefore evaluate a policy frontier that, for each \ExclusionRate, shows the maximum aid delivered per eligible recipient.

In addition to our proposed SAFE method, we compare against several benchmark strategies. Two natural baselines guarantee full coverage. The first, which we call ``Give all,'' divides the transfer budget evenly across all neighborhoods in the country. As a result, no neighborhood (whether eligible or ineligible) is excluded from receiving aid, but a substantial portion of the budget is misallocated to ineligible recipients. The second baseline, ``Survey all,'' takes the opposite approach by conducting a complete census to identify every eligible neighborhood before distributing aid. This eliminates misallocation but incurs substantial survey costs, reducing the budget available for transfers. To ensure a fair comparison with SAFE, both baselines drops a fraction of neighborhoods sampled uniformly at random, so that all methods operate under the same allowable exclusion rate.
We also evaluate ``CQR,'' which uses the SQR-EO model together with standard conformal quantile regression to construct prediction intervals with $\alpha$ coverage and classifies neighborhoods as described in Section~\ref{sec:ct_res}. Finally, we include an ``Oracle'' strategy, representing the best achievable outcome under the assumption that each neighborhood's eligibility status is known in advance.

\begin{figure}[!htbp]
    \centering
    \includegraphics[width=0.75\linewidth]{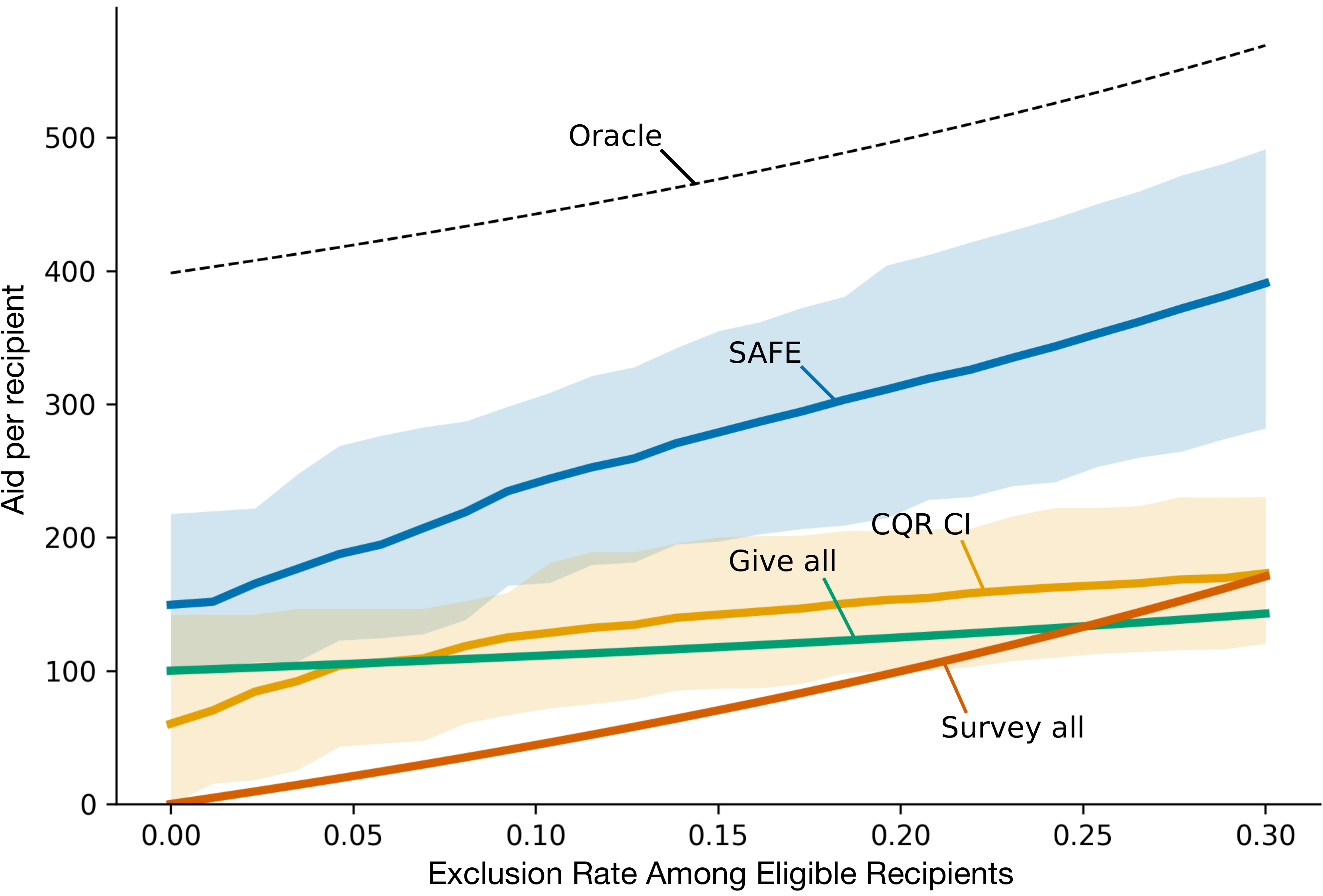}
    \caption{\textbf{Aid-allocation frontiers under varying exclusion tolerance.} Curves show the mean aid delivered per eligible recipient across 17 countries as the maximum permitted share of eligible neighborhoods left without aid varies from 0 to 0.30. Shaded regions span the lowest- and highest-performing countries. SAFE achieves the highest return among the implementable strategies throughout and approaches the unattainable ``Oracle'' benchmark as the permitted exclusion rate increases.}
    \label{fig:pareto_fronts}
\end{figure}

Figure~\ref{fig:pareto_fronts} compares the policy frontiers, with shaded regions showing the range between the best- and worst-performing countries at each exclusion rate. Across the full range of target rates, $\Safe$ delivers more aid per eligible recipient on average than the other implementable strategies. Part of the increase in aid as the target rate rises is mechanical: when fewer eligible neighborhoods receive transfers, the available budget is divided among fewer recipients, as illustrated by the ``Give all'' strategy. A higher tolerance for exclusion can also reduce targeting costs. For example, ``Survey all'' requires fewer surveys as the permitted exclusion rate increases and therefore overtakes ``Give all'' at higher target rates. More generally, a less stringent target allows the uncertainty-aware procedures to rely more heavily on EO-ML predictions and less on costly surveys. As a result, $\Safe$ approaches the performance of the ``Oracle'' as the permitted \ExclusionRate increases.

\begin{figure}[!htbp]
    \centering
    \includegraphics[width=\linewidth]{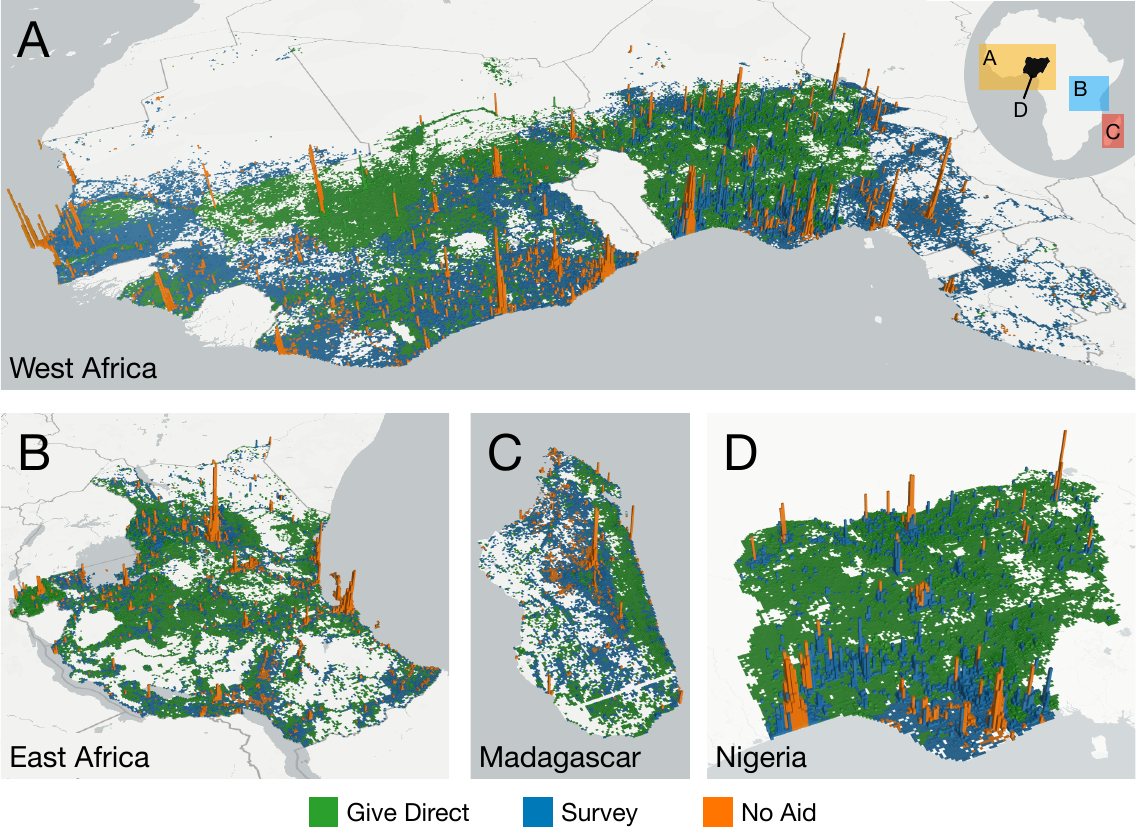}
    \caption{\textbf{Spatial allocation decisions produced by SAFE.} SAFE is applied to populated raster cells in countries with DHS surveys completed after 2020, using the most recent national survey as the calibration set. Bar height is proportional to cell population, while color indicates whether aid is given directly, withheld, or preceded by a follow-up survey. In general, aid tends to be allocated directly to poorer rural areas, withheld from dense urban centers, and directed through surveys in more ambiguous areas, including many smaller cities. These patterns vary locally: panel \textbf{D}, for example, shows sparsely populated rural cells in the relatively wealthier South-East Nigeria being assigned surveys rather than direct aid.}
    \label{fig:SAFE_map_panel}
\end{figure}

Figure~\ref{fig:SAFE_map_panel} illustrates the spatial allocation produced by SAFE with an exclusion rate of $0.05$, using the most recent available survey conducted in or after 2020 for each DHS country. Because the survey enumeration-area boundaries are unavailable, we apply SAFE to a raster grid, as with the maps presented in Figure~\ref{fig:iwi_and_ci_maps}, with bar height representing the population of each grid cell. The procedure allocates aid directly where a cell can be classified as eligible with sufficient confidence, withholds aid where it can be classified as ineligible, and directs follow-up surveys to cells for which the available evidence is inconclusive. Overall, direct allocation is most common in poorer rural areas, while surveys occur frequently in urban and peri-urban areas and in some rural pockets. It's worth noting that this information is not explicitly given to the model, but a pattern which the model has learned from the data.

\section{Discussion}
% The two panels tell complementary stories. Predicted wealth is highest in North Africa, along the Nile valley, in South Africa, and around cities and transport corridors, while most of the sub-Saharan interior shows low predicted IWI. Uncertainty follows a different geography: prediction intervals are narrowest across the uniformly poor central belt and widest in the wealthier, more heterogeneous north and south. 
%Descriptively, narrow intervals coincide with the uniformly poor central belt, while wider intervals appear in the wealthier and more heterogeneous north and south. This pattern is consistent with interval width tracking local outcome variability rather than wealth level itself, although we have not formally decomposed the drivers of width. 

EO-ML poverty mapping has reached point-prediction accuracy that makes it tempting to use model outputs directly for policy decisions \cite{Jean2016,Yeh2020,chi_microestimates_2022,pettersson2023,Wang2024}. This study shows both why that temptation is risky and why it still remains promising. Although our model's out-of-sample fit matches the upper end of prior work, conformalized quantile regression produces prediction intervals that remain wide for most neighborhoods. Interval widths are fundamentally determined by the error of the underlying predictions. Thus, other EO-ML poverty models---which have similar or lower accuracy---likely entail comparable uncertainty without explicitly quantifying it.

The practical implication is not that EO-ML estimates should be discarded, but rather that their uncertainty must be incorporated into policy decisions. CT and SAFE methods do this by allowing for abstention when the evidence is insufficient. They act on EO-ML predictions when the signal is sufficiently clear and defer to follow-up measurement when it is not, maintaining statistical guarantees for the vulnerable population. The goal is not to pitch EO-ML estimates and surveys against each other, but to show how they, when combined, can lead to policy decisions which are both safe and cost-effective.

The proposed calibration and decision procedures do not depend on satellite imagery, asset wealth, or a transformer architecture. They can be paired with any predictor that supplies suitable scores or quantiles and with a labeled calibration sample that is exchangeable with the intended deployment population. Possible inputs therefore include mobile-phone, administrative, or survey-linked data in addition to EO imagery \citep{blumenstock_estimating_2018}, and the same recipe extends to other EO-ML products that inform action, such as gridded population mapping and other screening problems \cite{tatem_worldpop_2017}. Phone-based poverty predictions have already informed humanitarian targeting at national scale \citep{aiken_machine_2022}. Distribution-free calibration could thus provide a common interface with controlled error rates and transparent abstention, something more important for planners than a higher $R^2$.

The guarantees provided by conformal prediction depend on the calibration data being exchangeable with the population in which the model is deployed. When a program expands to a new domain, a calibration set drawn from an earlier context will likely no longer represent the current errors \citep{angelopoulos2023conformal}. In practice, the most robust path is likely to budget for a small, recent calibration survey in each deployment region and to refresh it as conditions change. Importantly, the guarantees are marginal over the entire population, and thus do not necessarily hold for all policy-relevant subgroups. If calibration data for these groups exists, this could likely be remedied by combining CT and SAFE with class-conditional conformal prediction, as described by \citet{angelopoulos2023conformal}, but we leave this as future work. % An interesting future avenue of research would be an ML assisted online version of the aid-allocation problem, where a model is used for \mbp{...}

A further limitation is that SAFE treats eligibility as a binary status determined by a fixed poverty threshold. This formulation assigns the same importance to all classification errors of a given type, even though their welfare consequences may differ substantially. For example, allocating aid to a neighborhood just above the threshold may arguable bring greater utility than conducting an additional survey. A useful extension would therefore replace binary eligibility with a continuous, policy-specific utility function that accounts for the severity of poverty, while still maintaining the coverage guarantee for neighborhoods below the threshold. Incorporating such a utility into SAFE, following related utility-based approaches to aid targeting like in \citet{aiken_machine_2022}, could direct surveys toward cases in which resolving uncertainty has the greatest expected welfare benefit.

Taken together, our findings show that strong predictive performance alone is not sufficient for responsible policy use of EO-ML poverty estimates. By making uncertainty explicit and directing follow-up surveys to cases in which the available evidence is insufficient, CT and SAFE combine the scale of EO-ML with the reliability of conventional measurement. This provides a practical path from accurate poverty mapping toward accountable, uncertainty-aware decision-making.

\subsection*{Acknowledgments and Disclosure of Funding}

Computational resources were provided by the National Academic Infrastructure for Supercomputing in Sweden (NAISS), funded by the Swedish Research Council. MBP, JB and AD were supported by SRC grant numbers 2020-03088 and 2020-00491, as well as the European Horizon project ``ToBe -- Towards a sustainable wellbeing economy: integrated policies and transformative indicators'' (grant agreement number 101094211).

%TC:ignore

\clearpage
\bibliographystyle{plainnat}
\bibliography{main}

@article{Smits2015,
   author = {Jeroen Smits and Roel Steendijk},
   doi = {10.1007/s11205-014-0683-x},
   issn = {1573-0921},
   issue = {1},
   journal = {Social Indicators Research},
   pages = {65-85},
   title = {The {International} {Wealth} {Index} ({IWI})},
   volume = {122},
   url = {https://doi.org/10.1007/s11205-014-0683-x},
   year = {2015}
}

@article{Kingma2014,
   author = {Diederik P. Kingma and Jimmy Ba},
   title = {Adam: A Method for Stochastic Optimization},
   url = {http://arxiv.org/abs/1412.6980},
   journal={arXiv:1412.6980},
   year = {2014}
}

@article{Yeh2020,
   author = {Christopher Yeh and Anthony Perez and Anne Driscoll and George Azzari and Zhongyi Tang and David Lobell and Stefano Ermon and Marshall Burke},
   doi = {10.1038/s41467-020-16185-w},
   issn = {20411723},
   issue = {1},
   journal = {Nature Communications},
   title = {Using publicly available satellite imagery and deep learning to understand economic well-being in {Africa}},
   volume = {11},
   year = {2020}
}

@article{Gorelick2017,
   author = {Noel Gorelick and Matt Hancher and Mike Dixon and Simon Ilyushchenko and David Thau and Rebecca Moore},
   doi = {10.1016/j.rse.2017.06.031},
   issn = {00344257},
   journal = {Remote Sensing of Environment},
   title = {Google {Earth} {Engine}: Planetary-scale geospatial analysis for everyone},
   volume = {202},
   year = {2017}
}

@techreport{Burgert2013,
    author={Burgert, Clara R. and Colston, Josh and Roy, Thea and Zachary, Blake},
    title={Geographic displacement procedure and georeferenced data release policy for the {Demographic} and {Health} {Surveys}},
    series={{DHS} Spatial Analysis Reports No. 7},
    year={2013},
    publisher={ICF International},
    institution={ICF International},
    address={Calverton, Maryland, USA},
    url={http://dhsprogram.com/pubs/pdf/SAR7/SAR7.pdf}
}

@article{Jean2016,
   author = {Neal Jean and Marshall Burke and Michael Xie and W. Matthew Davis and David B. Lobell and Stefano Ermon},
   doi = {10.1126/science.aaf7894},
   issn = {10959203},
   issue = {6301},
   journal = {Science},
   title = {Combining satellite imagery and machine learning to predict poverty},
   volume = {353},
   year = {2016}
}

@misc{dhs,
  author = {{DHS Program}},
  title = {{Demographic} and {Health} {Surveys}},
  year = 2026,
  howpublished = {\url{www.dhsprogram.com}}
}

@article{tatem_worldpop_2017,
    author={Tatem, Andrew J.},
    title={{WorldPop}, open data for spatial demography},
    journal={Scientific Data},
    year={2017},
    volume={4},
    number={1},
    pages={170004},
    issn={2052-4463},
    doi={10.1038/sdata.2017.4},
    url={https://doi.org/10.1038/sdata.2017.4}
}

@article{groves_total_2010,
	title = {Total Survey Error: Past, Present, and Future},
	volume = {74},
	issn = {0033-362X},
	shorttitle = {Total {Survey} {Error}},
	url = {https://academic.oup.com/poq/article/74/5/849/1817502},
	doi = {10.1093/poq/nfq065},
	number = {5},
	journal = {Public Opinion Quarterly},
	author = {Groves, Robert M. and Lyberg, Lars},
	year = {2010},
	pages = {849--879}
}

@book{lavrakas_encyclopedia_2008,
	address = {2455 Teller Road, Thousand Oaks California 91320 United States of America},
	title = {Encyclopedia of {Survey} {Research} {Methods}},
	isbn = {978-1-4129-1808-4 978-1-4129-6394-7},
	url = {http://methods.sagepub.com/reference/encyclopedia-of-survey-research-methods},
	publisher = {Sage Publications, Inc.},
	author = {Lavrakas, Paul},
	year = {2008},
	doi = {10.4135/9781412963947}
}

@article{chi_microestimates_2022,
	title = {Microestimates of wealth for all low- and middle-income countries},
	volume = {119},
	issn = {0027-8424, 1091-6490},
	url = {http://www.pnas.org/lookup/doi/10.1073/pnas.2113658119},
	doi = {10.1073/pnas.2113658119},
	number = {3},
	journal = {Proceedings of the National Academy of Sciences},
	author = {Chi, Guanghua and Fang, Han and Chatterjee, Sourav and Blumenstock, Joshua E.},
	year = {2022},
	pages = {e2113658119}
}

@article{blumenstock_estimating_2018,
	title = {Estimating Economic Characteristics with Phone Data},
	volume = {108},
	issn = {2574-0768},
	url = {https://www.aeaweb.org/doi/10.1257/pandp.20181033},
	doi = {10.1257/pandp.20181033},
	journal = {AEA Papers and Proceedings},
	author = {Blumenstock, Joshua E.},
	year = {2018},
	pages = {72--76}
}

@misc{dhs_harmonization,
  author={Ekbrand, Hans},
  year={2026},
  title={{DHSharmonisation}},
  howpublished = {\url{https://bitbucket.org/hansekbrand/dhsharmonisation/}},
  note = {Software package}
}

@inproceedings{
    satmae2022,
    title={Sat{MAE}: Pre-training Transformers for Temporal and Multi-Spectral Satellite Imagery},
    author={Yezhen Cong and Samar Khanna and Chenlin Meng and Patrick Liu and Erik Rozi and Yutong He and Marshall Burke and David B. Lobell and Stefano Ermon},
    booktitle={Advances in Neural Information Processing Systems},
    editor={Alice H. Oh and Alekh Agarwal and Danielle Belgrave and Kyunghyun Cho},
    year={2022},
    url={https://openreview.net/forum?id=WBhqzpF6KYH}
}

@inproceedings{sqr2018,
    title={Single-Model Uncertainties for Deep Learning},
    author={Natasa Tagasovska and David Lopez-Paz},
    booktitle={Neural Information Processing Systems},
    year={2018},
    url={https://api.semanticscholar.org/CorpusID:202539625}
}

@inproceedings{
    pettersson2023,
    author       = {Markus B. Pettersson and
                  Mohammad Kakooei and
                  Julia Ortheden and
                  Fredrik D. Johansson and
                  Adel Daoud},
    title        = {Time Series of Satellite Imagery Improve Deep Learning Estimates of Neighborhood-Level Poverty in {Africa}},
    booktitle    = {Proceedings of the Thirty-Second International Joint Conference on Artificial Intelligence, {IJCAI-23}},
    pages        = {6165--6173},
    publisher    = {International Joint Conferences on Artificial Intelligence Organization},
    year         = {2023},
    url          = {https://doi.org/10.24963/ijcai.2023/684},
    doi          = {10.24963/ijcai.2023/684}
}

@INPROCEEDINGS{mae2022,
  author={He, Kaiming and Chen, Xinlei and Xie, Saining and Li, Yanghao and Dollár, Piotr and Girshick, Ross},
  booktitle={2022 IEEE/CVF Conference on Computer Vision and Pattern Recognition (CVPR)}, 
  title={Masked Autoencoders Are Scalable Vision Learners}, 
  year={2022},
  volume={},
  number={},
  pages={15979-15988},
  doi={10.1109/CVPR52688.2022.01553}
}

@article{fox1964,
author = {Martin Fox and Herman Rubin},
title = {Admissibility of Quantile Estimates of a Single Location Parameter},
volume = {35},
journal = {The Annals of Mathematical Statistics},
number = {3},
publisher = {Institute of Mathematical Statistics},
pages = {1019--1030},
year = {1964},
doi = {10.1214/aoms/1177700518},
URL = {https://doi.org/10.1214/aoms/1177700518}
}

@Article{Li2020,
author={Li, Xuecao
and Zhou, Yuyu
and Zhao, Min
and Zhao, Xia},
title={A harmonized global nighttime light dataset 1992--2018},
journal={Scientific Data},
year={2020},
volume={7},
number={1},
pages={168},
issn={2052-4463},
doi={10.1038/s41597-020-0510-y},
url={https://doi.org/10.1038/s41597-020-0510-y}
}

@inproceedings{
dosovitskiy2021,
title={An Image is Worth 16x16 Words: Transformers for Image Recognition at Scale},
author={Alexey Dosovitskiy and Lucas Beyer and Alexander Kolesnikov and Dirk Weissenborn and Xiaohua Zhai and Thomas Unterthiner and Mostafa Dehghani and Matthias Minderer and Georg Heigold and Sylvain Gelly and Jakob Uszkoreit and Neil Houlsby},
booktitle={International Conference on Learning Representations},
year={2021},
url={https://openreview.net/forum?id=YicbFdNTTy}
}

@book{vovk2005,
  title = {Algorithmic Learning in a Random World},
  author = {Vovk, Vladimir and Gammerman, Alexander and Shafer, Glenn},
  year = {2005},
  publisher = {Springer},
  location = {New York, NY},
  doi = {10.1007/b106715},
  url = {http://link.springer.com/10.1007/b106715},
  isbn = {978-0-387-00152-4},
  pagetotal = {XVI, 324}
}

@misc{landsat8,
  author       = {{Earth Resources Observation and Science (EROS) Center}},
  title        = {Landsat 8--9 {Operational} {Land} {Imager} / {Thermal} {Infrared} {Sensor} {Level}-2, {Collection} 2},
  year         = {2020},
  publisher    = {U.S. Geological Survey},
  doi          = {10.5066/P9OGBGM6},
  url          = {https://doi.org/10.5066/P9OGBGM6},
  note         = {Dataset}
}

@misc{landsat7,
  author       = {{Earth Resources Observation and Science (EROS) Center}},
  title        = {Landsat 7 {Enhanced} {Thematic} {Mapper} {Plus} {Level}-2, {Collection} 2},
  year         = {2020},
  publisher    = {U.S. Geological Survey},
  doi          = {10.5066/P9C7I13B},
  url          = {https://doi.org/10.5066/P9C7I13B},
  note         = {Dataset}
}

@misc{landsat45,
  author       = {{Earth Resources Observation and Science (EROS) Center}},
  title        = {Landsat 4-5 {Thematic} {Mapper} {Level}-2, {Collection} 2},
  year         = {2020},
  publisher    = {U.S. Geological Survey},
  doi          = {10.5066/P9IAXOVV},
  url          = {https://doi.org/10.5066/P9IAXOVV},
  note         = {Dataset}
}

@inproceedings{cqr_romano,
 author = {Romano, Yaniv and Patterson, Evan and Candes, Emmanuel},
 booktitle = {Advances in Neural Information Processing Systems},
 editor = {H. Wallach and H. Larochelle and A. Beygelzimer and F. d\textquotesingle Alch\'{e}-Buc and E. Fox and R. Garnett},
 pages = {},
 publisher = {Curran Associates, Inc.},
 title = {Conformalized Quantile Regression},
 url = {https://proceedings.neurips.cc/paper\_files/paper/2019/file/5103c3584b063c431bd1268e9b5e76fb-Paper.pdf},
 volume = {32},
 year = {2019}
}

@article{Wang2024,
author = {Mengjie Wang and Xi Li},
title = {Estimating asset wealth using multidimensional luminous information in areas lacking nighttime light},
journal = {International Journal of Digital Earth},
volume = {17},
number = {1},
pages = {2336049},
year = {2024},
publisher = {Taylor \& Francis},
doi = {10.1080/17538947.2024.2336049},
URL = {https://doi.org/10.1080/17538947.2024.2336049},
eprint = {https://doi.org/10.1080/17538947.2024.2336049}
}

@article{sherman2026global,
  title={Global high-resolution estimates of the {UN} {Human} {Development} {Index} using satellite imagery and machine learning},
  author={Sherman, Luke and Proctor, Jonathan and Druckenmiller, Hannah and Tapia, Heriberto and Hsiang, Solomon},
  journal={Nature Communications},
  volume={17},
  number={1},
  pages={1315},
  year={2026},
  publisher={Nature Publishing Group UK London}
}

@article{zheng2025dynamic,
  title={Dynamic, high-resolution poverty measurement in data-scarce environments},
  author={Zheng, Zhuo and Wu, Timothy and Lee, Richard and Newhouse, David and Kilic, Talip and Burke, Marshall and Ermon, Stefano and Lobell, David B},
  journal={Journal of Development Economics},
  pages={103691},
  year={2025},
  publisher={Elsevier}
}

@article{marty2024global,
  title={Global poverty estimation using private and public sector big data sources},
  author={Marty, Robert and Duhaut, Alice},
  journal={Scientific Reports},
  volume={14},
  number={1},
  pages={3160},
  year={2024},
  publisher={Nature Publishing Group UK London}
}

@article{angelopoulos2023conformal,
  title={Conformal prediction: A gentle introduction},
  author={Angelopoulos, Anastasios N and Bates, Stephen},
  journal={Foundations and Trends in Machine Learning},
  volume={16},
  number={4},
  pages={494--591},
  year={2023},
  publisher={Emerald Publishing Limited}
}

@techreport{sahoo2025would,
  title={What Would it Cost to End Extreme Poverty?},
  author={Sahoo, Roshni and Blumenstock, Joshua and Niehaus, Paul and Selker, Leo and Wager, Stefan},
  year={2025},
  institution={National Bureau of Economic Research}
}

@article{aiken_machine_2022,
  title   = {Machine learning and phone data can improve targeting of humanitarian aid},
  author  = {Aiken, Emily and Bellue, Suzanne and Karlan, Dean and Udry, Chris and Blumenstock, Joshua E.},
  journal = {Nature},
  volume  = {603},
  number  = {7903},
  pages   = {864--870},
  year    = {2022},
  doi     = {10.1038/s41586-022-04484-9}
}

@misc{worldpop_r2025a,
  author={Bondarenko, Maksym and Priyatikanto, Rhorom and Tejedor-Garavito, Natalia and Zhang, Wei and McKeen, Tessa and Cunningham, Andrew and Woods, Thomas and Hilton, Jason and Cihan, Derya and Nosatiuk, Bogdan and Brinkhoff, Thomas and Tatem, Andrew and Sorichetta, Alessandro},
  title={Constrained estimates of 2015--2030 total number of people per grid square at a resolution of 3 arc (approximately 100 m at the equator), R2025A version v1},
  year={2025},
  howpublished={WorldPop, School of Geography and Environmental Science, University of Southampton},
  doi={10.5258/SOTON/WP00839},
  url={https://hub.worldpop.org/doi/10.5258/SOTON/WP00839}
}

@techreport{Elkasabi2020,
  author      = {Elkasabi, Mahmoud and Ren, Ruilin and Pullum, Thomas W.},
  title       = {Multilevel Modeling Using {DHS} Surveys: A Framework to Approximate Level-Weights},
  institution = {ICF},
  type        = {{DHS} Methodological Reports},
  number      = {27},
  year        = {2020},
  address     = {Rockville, Maryland, USA},
  url         = {https://www.dhsprogram.com/pubs/pdf/MR27/MR27.pdf}
}

@article{gasparin2024merging,
  title={Merging uncertainty sets via majority vote},
  author={Gasparin, Matteo and Ramdas, Aaditya},
  journal={arXiv preprint arXiv:2401.09379},
  year={2024}
}

@article{aikenWhenShouldBig2025,
  title = {When Should Big Data and Algorithms Be Used to Determine Programme Eligibility?},
  author = {Aiken, Emily and Ashraf, Anik and Blumenstock, Joshua and Guiteras, Raymond and Mushfiq Mobarak, Ahmed and Hu, Nicole},
  year = {2025},
  journal = {VoxDev},
  url = {https://voxdev.org/topic/methods-measurement/when-should-big-data-and-algorithms-be-used-determine-programme}
}

\clearpage
\appendix

% Number appendix floats independently from those in the main manuscript.
\setcounter{figure}{0}
\setcounter{table}{0}
\renewcommand{\figurename}{Supplementary Figure}
\renewcommand{\tablename}{Supplementary Table}

\section{Detailed methods and technical setup}
\label{app:detailed_methods}

\subsection{SQR EO-ML model}
%\section{EO-ML model details}
\label{sec:model_details}

\jbl{Change the use of $\alpha$ in the appendix so that it matches how we use it in the main text: as the desired coverage rate (e.g., 95\%), rather than the desired error rate (e.g., 5\%)}

\subsubsection{Input construction}

We use Landsat 5, 7, and 8 multispectral imagery \citep{Gorelick2017}. For each neighborhood location, we extract $224\times224$ patches (approximately $6.72\times6.72$ km at 30 m resolution) with six channels: Blue, Green, Red, NIR1, NIR2, and SWIR. The patch size is chosen to match the neighborhood-scale unit used for DHS-cluster prediction and to remain comparable with prior EO-ML work \citep{pettersson2023, Yeh2020}.

Landsat revisits are frequent, about once every 16th day, but many frames are unusable due to cloud contamination \citep{landsat8}. For each location, we therefore select up to 25 low-cloud Landsat frames from the historical archive and additionally include the least-cloudy frame from the year preceding the survey date. This yields computationally tractable sequences while preserving both long-run context and recent pre-survey signal.

We use harmonized nighttime-light data that bridges DMSP (1992-2013) and VIIRS (2012 onward) to obtain a comparable temporal signal \citep{Li2020}. Because native nighttime-light resolution is coarse (about 1 km/pixel), each sample covers a $7 \times 7$ nighttime-light patch over the same area as the Landsat crop.

\subsubsection{Model architecture}

The model has a spatial encoder, a temporal encoder, and a scalar prediction head. Each Landsat frame is encoded with a ViT-based spatial encoder, while each nighttime-light frame is projected to the same embedding size using a linear layer \citep{dosovitskiy2021}. Temporal encodings are then added to all frame tokens before temporal aggregation.

\begin{figure}[!htbp]
    \centering
    \includegraphics[width=\textwidth]{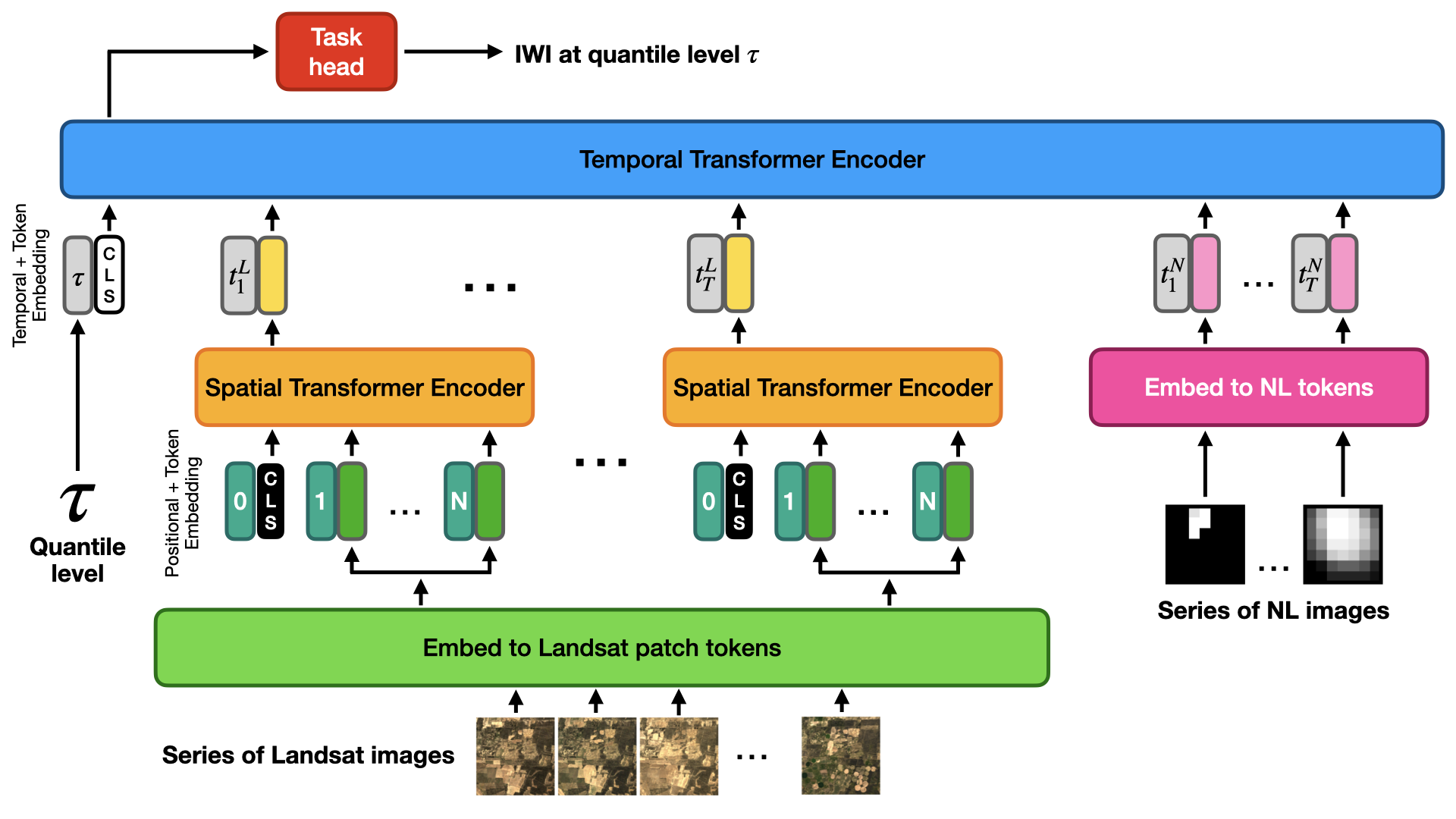}
    \caption{Architecture of the EO-ML model. Landsat frames are encoded by a spatial transformer, combined with nighttime-light images and conditioning tokens in a temporal transformer, and mapped to a scalar IWI output.}
    \label{fig:model_architecture}
\end{figure}

As EO sequences are irregularly sampled, we use explicit time encodings instead of relying only on input order. Following the design of SATMAE by \citet{satmae2022}, each frame is encoded using year, month, and hour components, where year is represented relative to the survey year (to avoid leakage from absolute calendar time). The sinusoidal base encoding is:
\begin{equation*}
\mathrm{Encode}(k,2i)=\sin\left(\frac{k}{\Omega^{2i/d}}\right),\qquad
\mathrm{Encode}(k,2i+1)=\cos\left(\frac{k}{\Omega^{2i/d}}\right),
\end{equation*}
and the frame-level temporal feature is constructed by concatenating encoded year, month, and hour components.

The temporal transformer receives these frame tokens together with a dedicated $\tau$-token used for quantile conditioning. The final contextual representation (via the transformer output token) is passed through a linear task head to produce a scalar IWI estimate at quantile level $\tau$.

\subsubsection{Self-supervised pretraining}

To improve representation quality under limited labeled survey data, the spatial Landsat encoder is initialized from masked-autoencoder (MAE) pretraining \citep{mae2022}. Pretraining uses approximately 300,000 additional unlabeled locations across Africa, sampled with population weighting based on WorldPop \cite{tatem_worldpop_2017}. A high masking ratio (75\%) is used during MAE training on single-frame Landsat inputs; learned weights are then transferred to the supervised IWI model.

\subsubsection{Optimization and inference settings}

The supervised model is optimized with AdamW (learning rate $10^{-4}$, effective batch size 16) for up to 200 epochs \citep{Kingma2014}. During training, we subsample frame sequences to increase variation (average of about 10 Landsat frames and 5 nighttime-light frames per sample) and apply random flip augmentation. The best-validation checkpoint is retained.

At inference, all available selected frames are used (up to 25 Landsat frames per sample). The model is queried at $\tau\in\{0.00, 0.01, ..., 0.99, 1.00\}$ for predictions at each percentile. The $\tau=0.5$ output is used as the point estimate and $\tau\in\{0.05,0.95\}$ define the nominal 90\% model-based interval before conformal calibration.

\begin{figure}[!htbp]
    \centering
    \includegraphics[width=0.5\linewidth]{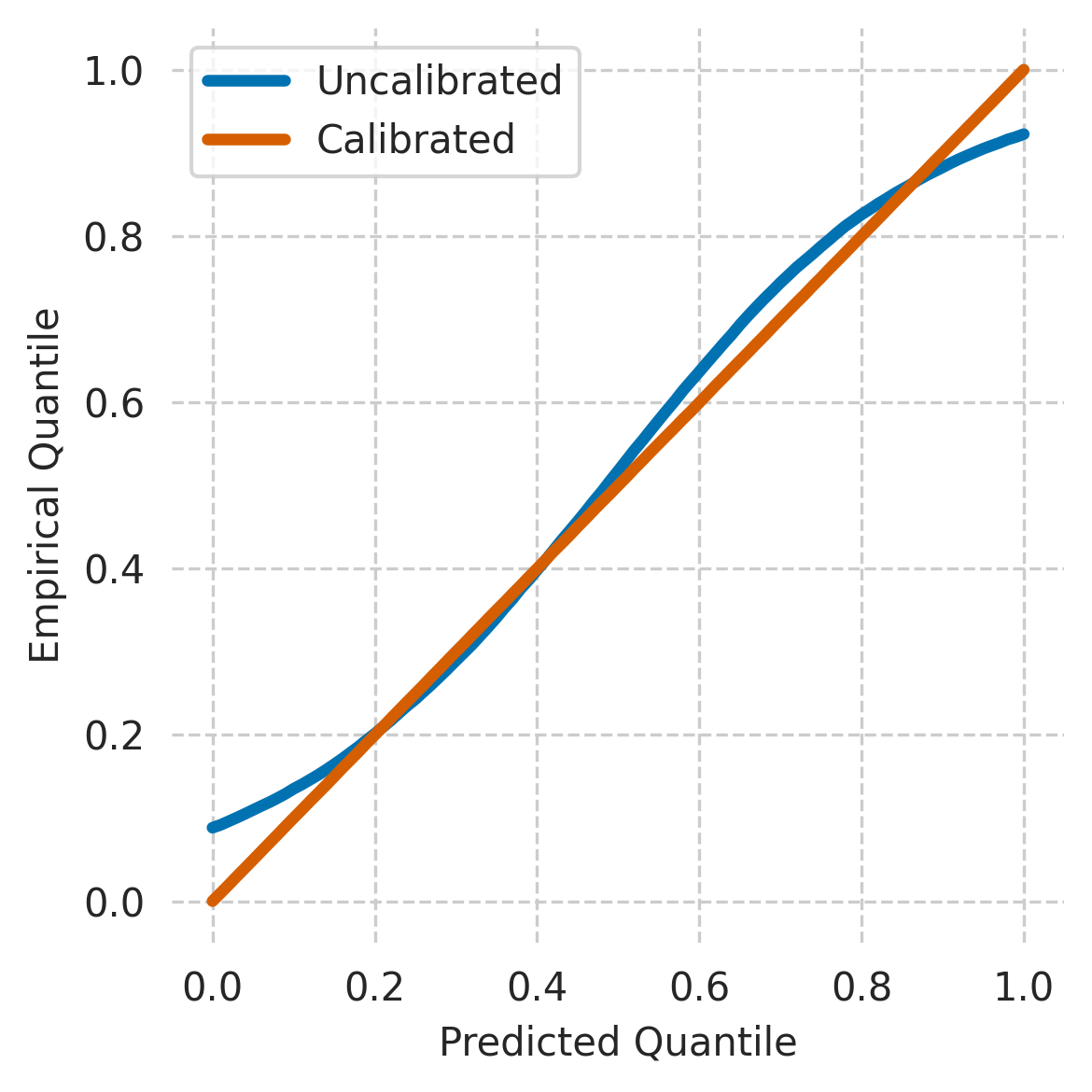}
    \caption{QQ-plot of the model predicted against the observed quantiles before and after conformal prediction calibration. The naive prediction intervals capture much of the uncertainty but undercover: the nominal 90\% interval contains the observed outcome for 79.5\% of held-out neighborhoods. After conformal calibration, coverage matches the target (90.0\%), in line with the theoretical guarantees.}
    \label{fig:qq_plot}
\end{figure}

\subsubsection{Simultaneous quantile regression}

To produce uncertainty-aware predictions, we train the EO-ML model as a simultaneous-quantile regressor rather than only as a conditional-mean predictor. For a fixed quantile level $\tau \in (0,1)$, the standard objective is the pinball loss%\jb{Removing white space so that indenting of text below equations and whitespace above align environment is corrected}
\begin{equation*}
\ell_\tau(y, \hat{y})=
\begin{cases}
\tau(y-\hat{y}) & \text{if } y-\hat{y} \ge 0,\\
(1-\tau)(\hat{y}-y) & \text{otherwise.}
\end{cases}
\end{equation*}
Minimizing this loss yields an estimate of the conditional $\tau$-quantile of $Y\mid X$ \citep{fox1964}. A direct approach would train a separate model for each quantile of interest, but that is inefficient and can produce quantile crossing, where the outcome does not increase monotonically with the quantile \citep{sqr2018}.

Instead, we use simultaneous quantile regression (SQR), %\jb{removing capitals to be consistent with the main text}
as introduced by \citet{sqr2018}, where the model takes both covariates $x$ and a quantile index $\tau$ as input and is trained over the full range of quantiles. The objective is
\begin{equation*}
\hat{f} \in \arg\min_f \frac{1}{n}\sum_{i=1}^{n}
\mathbb{E}_{\tau\sim U[0,1]}\left[\ell_\tau\big(f(x_i,\tau),y_i\big)\right].
\end{equation*}
In practice, this expectation is approximated stochastically by sampling $\tau$ values for each sample-iteration during training. In our architecture, this is implemented by conditioning the temporal transformer on a dedicated $\tau$-token, so one shared network is trained to produce quantile estimates across the full distribution. % (see Appendix~\ref{sec:model_details} for details).\jb[this reference is now self-referential]

At inference time, if we wish to evaluate a 90\% prediction interval, we make two predictions fixing $\tau$ as $0.05$ and $0.95$, obtaining a lower and upper bound of a nominal 90\% model-based interval. These intervals are informative but remain model-dependent. Coverage is therefore not guaranteed at this stage and will have to be further calibrated in the conformal step below.

\subsection{Conformalized quantile regression}
%\subsection{Conformalized quantile regression (vanilla)}

In order to calibrate the nominal prediction intervals produced by the SQR model, we use conformalized quantile regression (CQR), as introduced by \citet{cqr_romano}, on a held-out calibration set that is not used for model fitting. Under exchangeability between calibration and test points, conformal prediction guarantees marginal coverage close to the target level. For target miscoverage $\alpha$, the interval $\hat{C}(x)$ satisfies
\begin{equation*}%\jb{Removing white space so that indenting of text below equations and whitespace above align environment is corrected}
1 - \alpha \leq \mathbb{P}\left(y_{\mathrm{test}}\in \hat{C}(x_{\mathrm{test}})\right) \leq 1-\alpha + \frac{1}{n+1},
\end{equation*}
where $n$ is the number of calibration points. Starting from model quantiles $\hat{t}_{\alpha/2}(x)$ and $\hat{t}_{1-\alpha/2}(x)$, we compute nonconformity scores on calibration samples:
\begin{equation*}
s(x_i,y_i) = \max\left\{\hat{t}_{\alpha/2}(x_i) - y_i, \; y_i - \hat{t}_{1 - \alpha/2}(x_i)\right\}.
\end{equation*}

Let $\hat{q}$ be the empirical $\left\lceil (n+1)(1 - \alpha)\right\rceil / n$ quantile of these scores. The calibrated interval is then
\begin{equation*}
\hat{C}(x) = \left[\hat{t}_{\alpha/2}(x) - \hat{q}, \; \hat{t}_{1 - \alpha/2}(x) + \hat{q}\right].
\end{equation*}

Intuitively, conformal calibration expands or contracts the model-based intervals using the magnitude of held-out residuals, yielding valid marginal coverage while retaining location-specific variation from the EO-ML model.

%% [Fable_v3] The conformal-pipeline figure was moved to the Method overview in the main text
%% (Adel's instruction); see fig:cp_inference_pipeline there.

\subsection{Reliable EO-ML methods}

In many applications, the primary objective is not to predict the exact regression value, but rather to determine whether it lies above or below a fixed threshold $\beta$. Examples include assessing whether pollution levels exceed regulatory limits or whether the IWI of a neighborhood falls below a poverty threshold. A straightforward approach would be to train a dedicated classifier for each threshold, but this restricts the model to a specific choice of $\beta$. In contrast, prediction interval-based models are more flexible because the threshold can be selected or adjusted after training. Prediction intervals also provide a natural decision rule for thresholding: if the entire interval lies above $\beta$, the sample can be confidently classified as above the threshold; if the interval lies entirely below $\beta$, it can be classified as below. However, when the interval contains $\beta$, the data do not support either classification with the desired confidence. In such cases, a reliable method should abstain from assigning a class and instead label the sample as \emph{indeterminate}.

%% [Fable_v2] The thresholding-task figure (whose master caption read ``Enter Caption'') was moved to the
%% general-audience Method overview in the main text, with a drafted caption; see fig:thresholding_task there.

\subsubsection{Conformalized thresholding (CT)}

We implement reliable thresholding by extending CQR to threshold classification with set error rates. We refer to this procedure as Conformalized Thresholding (CT).

Taking inspiration from ``class-conditional conformal'' as defined by \citet{angelopoulos2023conformal}, CT uses separate conformal corrections for the two classes, i.e., the two sides of the threshold. Define the threshold scores%\jb{Removing white space so that indenting of text below equations and whitespace above align environment is corrected}
\begin{align*}
    s_{\leq}(x) &= \hat t_{\alpha_1} (x) - \beta, \\
    s_{>}(x) &= \beta - \hat t_{1-\alpha_2}(x).
\end{align*}
Let $\hat{q}_{\leq}$ be the $\lceil (n_1+1)(1-\alpha_1) \rceil / n_1$ quantile of the scores among calibration points at or below the threshold,
\begin{equation}\label{eqFirstCalibrationSet}
    \{s_{\leq} (X_i) \mid i \in \{1,\ldots,n\} \mathrm{\ s.t.\ } Y_i \le \beta\},
\end{equation} 
where $\{(X_i, Y_i)\}_{i=1}^n$ is the calibration set and $n_1$ is the cardinality of~\eqref{eqFirstCalibrationSet}. Similarly, let $\hat{q}_{>}$ be the $\lceil (n_2+1)(1-\alpha_2) \rceil / n_2$ quantile of the scores among calibration points above the threshold that are not classified as ``Above'' after the first calibration step,
\begin{equation}\label{eqSecondCalibrationSet}
    \{s_{>}(X_i) \mid i \in \{1,\ldots,n\} \mathrm{\ s.t.\ } Y_i > \beta \text{ and } \hat t_{\alpha_1}(X_i) - \hat q_{\leq} \le \beta\},
\end{equation}
where $n_2$ is the cardinality of~\eqref{eqSecondCalibrationSet}.

\begin{theorem}
\label{the:ct}
    Suppose that the calibration data points $(X_1, Y_1), \ldots, (X_n, Y_n)$ and the new data point $(X_{n+1}, Y_{n+1})$ are exchangeable. Further suppose that $\alpha_i \ge 1/n_i$ for $i = 1,2$. Then CT controls the two threshold error rates in the sense that
    \begin{align*}
        \Pr \left( \hat t_{\alpha_1}(X_{n+1}) - \hat{q}_{\leq} > \beta \mid Y_{n+1} \le \beta \right) &\le \alpha_1, \\
        \Pr \left( \hat t_{1-\alpha_2}(X_{n+1}) + \hat{q}_{>} < \beta \mid Y_{n+1} > \beta, \hat t_{\alpha_1}(X_{n+1}) - \hat q_{\leq} \le \beta \right) &\le \alpha_2.
    \end{align*}
\end{theorem}

The condition $\alpha_i \ge 1/n_i$ ensures that the corresponding quantile is well-defined, i.e., that $\lceil (n_i+1)(1-\alpha_i) \rceil / n_i \le 1$. The proof is provided in Appendix~\ref{sec:ct_proof}.

Given these calibrated corrections, we construct the CT classifier $\Thresh$ as
\[ \Thresh(x) = \begin{cases}
    \text{``Above''} &\text{if\ } \hat t_{\alpha_1}(x) - \hat q_{\leq} > \beta, \\
    \text{``Below''} &\text{else if } \hat t_{1-\alpha_2}(x) + \hat q_{>} < \beta, \\
    \text{``Indeterminate''} &\text{otherwise}.
\end{cases}\]

The intuition behind CT is that the two types of thresholding errors are inherently one-sided. A point with $Y \leq \beta$ can only be misclassified as ``Above'' if the lower bound is too large, while a point with $Y > \beta$ can only be misclassified as ``Below'' if the upper bound is too small. This allows the two error events to be calibrated separately.

To connect this rule to prediction intervals, define the one-sided prediction sets
$$
C_{\leq}(X_i) = \left[\hat t_{\alpha_1}(X_i) - \hat q_{\leq}, \, \infty \right]
$$
and
$$
C_{>}(X_i) = \left[-\infty, \, \hat t_{1-\alpha_2}(X_i) + \hat q_{>} \right].
$$

The set $C_{\leq}$ is calibrated using only calibration samples satisfying $Y_i \le \beta$. Consequently,
$$
\Pr \left(\beta \in C_{\leq}(X_{n+1}) \mid Y_{n+1} \leq \beta \right) \geq 1-\alpha_1.
$$
Intuitively, $C_{\leq}$ acts as a conservative lower bound: if $\beta \notin C_{\leq}(X_i)$, then the entire interval lies above the threshold, providing evidence that $Y_i > \beta$.

Similarly, $C_{>}$ is calibrated using only samples with $Y_i > \beta$ that are not classified as ``Above'', yielding the same false-positive guarantee for the second CT decision step,
$$
\Pr \left(\beta \in C_{>}(X_{n+1}) \mid Y_{n+1} > \beta \right) \geq 1-\alpha_2.
$$
Thus, $C_{>}$ acts as a conservative upper bound: if $\beta \notin C_{>}(X_i)$, then the interval lies entirely below the threshold, suggesting $Y_i \le \beta$.

The CT classifier combines these two one-sided sets. Equivalently, using the construction outlined above, one may consider the two-sided interval
$$
C(X_i)=C_{\leq}(X_i) \cap C_{>}(X_i).
$$
If $C(X_i)$ lies entirely above $\beta$, the point is classified as ``Above''; if it lies entirely below $\beta$, it is classified as ``Below''; otherwise, the interval overlaps the threshold and the prediction is declared ``Indeterminate.''
 
\subsubsection{SAFE decision making}

As discussed above, a major obstacle in eradicating poverty through monetary aid is figuring out which communities are in need of assistance. This information can be obtained through household surveys, but surveying every possible recipient is expensive and diverts resources away from transfers themselves. We operate on neighborhoods rather than individual households for three reasons: EO imagery resolves neighborhood-level conditions rather than single dwellings; the DHS releases cluster-level rather than household-level coordinates; and area-based targeting is an established first stage in practice, with within-community allocation handled by complementary mechanisms. To this end, we extend Conformalized Thresholding into an aid-allocation framework that we call \textit{Screen And Follow up with Error control}, or $\Safe$. The idea is to use our EO-ML model to screen out neighborhoods that are very unlikely to fall below a poverty threshold $\beta$, assign aid directly to those that likely do, and survey the remaining communities for which the model is not sufficiently certain.

In this setting, the key error is the eligible-unit exclusion rate: the share of neighborhoods truly below the poverty line that are mistakenly left without aid. CT allows this risk to be set in advance through the parameter $\alpha_1$. Controlling exclusion, however, is not sufficient to determine the full operational policy. Among neighborhoods that are not screened out, the algorithm must still decide which should receive aid directly and which should be sent to follow-up surveys. Direct aid risks spending resources on ineligible recipients, while surveys consume resources that could otherwise be used for transfers. The parameter $\alpha_2$ governs this trade-off by determining how conservative the rule is before assigning aid without a survey, and $\Safe$ selects the value of $\alpha_2$ that optimizes the resulting allocation policy for a fixed budget $T$.

Formally, consider a deployment set $\mathcal{D} := \{X_i\}_{i=1}^N$ of neighborhoods for which EO imagery is observed but true wealth outcomes are unobserved. The goal is to allocate aid to neighborhoods with $Y_i \le \beta$, where $\beta$ denotes the set poverty line. SAFE takes as input the trained SQR model $\hat{t}_{\tau}$, an exchangeable calibration set $\mathcal{C} := \{(X_j, Y_j)\}_{j=1}^n$, a utility function $\mathcal{U}$ to be maximized, a fixed budget $T$, and an exclusion-risk tolerance $\alpha_1$. It then chooses the CT parameter $\alpha_2$ that maximizes the estimated utility of the allocation policy. The full algorithm is provided in Appendix~\ref{sec:full_safe_algo}, but below we give a high-level intuition for the procedure:

\begin{algorithm}[ht]
\caption{SAFE calibration intuition}
\label{alg:safe_calibration_intuition}
\begin{algorithmic}[1]
\For{candidate values $\alpha_2 \in [0, 1]$}
    \State Calibrate a CT classifier $\Thresh$ on $\mathcal{C}$ using $(\alpha_1, \alpha_2)$
    \State $\hat{u} \gets$ estimated utility $\mathcal{U}$ of the allocation policy induced from $\Thresh$ under budget $T$.
    \If{$\hat{u} > u^*$}
        \State $u^* \gets \hat{u}$
        \State $\Thresh^* \gets \Thresh$
    \EndIf
\EndFor
\State Classify all neighborhoods in $\mathcal{D}$ with $\Thresh^*$
\State Survey all neighborhoods classified as ``Indeterminate''
\State Assign aid to neighborhoods classified as ``Below'' and to surveyed neighborhoods with $Y_i \le \beta$
\end{algorithmic}
\end{algorithm}

We leave the utility $\mathcal{U}$ purposefully vague to leave room for different things. Its value will likely be estimated using $\mathcal{C}$ and/or $\mathcal{D}$. As an example, consider when we wish to maximize the amount of aid per neighborhood with a fixed survey cost $c$ per neighborhood. In this setting, we have
\begin{equation*}%\label{aidPerUnitUtility}%\jb{removing label and equation numbering as it isn't used.}
    \mathcal{U} := \frac{T - cN_{\text{Survey}}}{N_{\text{Aid}}}.
\end{equation*}
This can be estimated from $\mathcal{C}$ and $\mathcal{D}$ with
$$
\begin{aligned}
    \hat{N}_{\text{Survey}} &= \hat{p}_{\text{Survey}} \cdot N, \\ 
    \hat{N}_{\text{Aid}} &= \hat{p}_{\text{Direct}} \cdot N +\hat{p}_{\text{Aid} | \text{Survey}} \cdot N_{\text{Survey}}.
\end{aligned}
$$
With the number of samples $N = |\mathcal{D}|$ and the remaining variables estimated from applying $\Thresh$ on $\mathcal{C}$.

%% [Fable_v3] The CT calibration walkthrough figure was moved to the Method overview in the
%% main text (Adel's instruction); see fig:SAFE_illustration there.

%% [Fable_v2] The full SAFE calibration algorithm appears once, in Appendix~\ref{sec:full_safe_algo}; the duplicated float that previously sat here (same \label used twice in the master) was removed.

\subsubsection{Proof of Theorem~\ref{the:ct}}
\label{sec:ct_proof}

\begin{proof}
    The CT procedure controls the false negative rate at the $\alpha_1$ level because
    \begin{align*}
        \Pr \left( \Thresh(X_{n+1}) = \text{``Above''} \mid Y_{n+1} \le \beta \right) &= \Pr \left( \hat t_{\alpha_1}(X_{n+1}) - \hat q_{\leq} > \beta \mid Y_{n+1} \le \beta \right) \\
        &= \Pr \left( \hat t_{\alpha_1}(X_{n+1}) - \beta > \hat q_{\leq} \mid Y_{n+1} \le \beta \right) \\
        &= \Pr \left( s_{\leq}(X_{n+1}) > \hat q_{\leq} \mid Y_{n+1} \le \beta \right) \\
        &\le \alpha_1,
    \end{align*}
    where the final line follows from exchangeability conditional on $Y_{n+1} \le \beta$ and the fact that $\hat q_{\leq}$ is the $\lceil (n_1+1)(1-\alpha_1) \rceil$-order statistic of~\eqref{eqFirstCalibrationSet}.

    Similarly, CT controls the false positive rate at the $\alpha_2$ level. Let
    \[
        A_{n+1}=\left\{\hat t_{\alpha_1}(X_{n+1}) - \hat q_{\leq} \le \beta\right\}
    \]
    denote the event that a wealthy point is not classified as ``Above'' by the first CT decision step. Then
    \begin{align*}
        \Pr ( \Thresh(&X_{n+1}) = \text{``Below''} \mid Y_{n+1} > \beta ) \\
        &= \Pr \left( \hat t_{\alpha_1}(X_{n+1}) - \hat q_{\leq} \le \beta, \hat t_{1-\alpha_2}(X_{n+1}) + \hat q_{>} < \beta \mid Y_{n+1} > \beta \right) \\
        &= \Pr \left( A_{n+1} \mid Y_{n+1} > \beta \right) \Pr \left( \hat t_{1-\alpha_2}(X_{n+1}) + \hat q_{>} < \beta \mid Y_{n+1} > \beta, A_{n+1} \right) \\
        &\le \Pr \left( \hat t_{1-\alpha_2}(X_{n+1}) + \hat q_{>} < \beta \mid Y_{n+1} > \beta, A_{n+1} \right) \\
        &= \Pr \left( \beta - \hat t_{1-\alpha_2}(X_{n+1}) > \hat q_{>} \mid Y_{n+1} > \beta, A_{n+1} \right) \\
        &= \Pr \left( s_{>}(X_{n+1}) > \hat q_{>} \mid Y_{n+1} > \beta, A_{n+1} \right) \\
        &\le \alpha_2,
    \end{align*}
    where the final line follows from exchangeability conditional on $Y_{n+1} > \beta$ and $A_{n+1}$ and the fact that $\hat q_{>}$ is the $\lceil (n_2+1)(1-\alpha_2) \rceil$-order statistic of~\eqref{eqSecondCalibrationSet}.
\end{proof}

\subsubsection{Full pseudocode for the SAFE calibration algorithm}
\label{sec:full_safe_algo}

\begin{algorithm}[H]
\caption{SAFE calibration}
\label{alg:safe_calibration}
\begin{algorithmic}[1]
\Require Calibration set $\mathcal{C}=\{(x_i,y_i)\}_{i=1}^n$, deployment set $\mathcal{D}=\{x_i\}_{i=1}^N$, threshold $\beta$, fixed FNR target $\alpha_1$, prediction models $\hat{t}_{\alpha}(\cdot)$ for all $\alpha \in [0, 1]$, utility estimator
$\widehat{\mathcal{U}}(\cdot;\mathcal C,\mathcal D)$
\Ensure Calibrated thresholding rule $\Threshstar$

% Phase 1
\State \texttt{- Phase 1: Lower Bound That Ensures FNR -}
\State $\mathcal{C}_{\leq \beta} \gets \{(x_i,y_i)\in\mathcal{C}: y_i \leq \beta\}$ % Only positive samples
\State $n_{\leq \beta} \gets |\mathcal{C}_{\leq \beta}|$

\State $\mathcal{S}_1 \gets \{ \tone(x_i)-\beta \mid (x_i, y_i)\in \mathcal{C}_{\leq \beta} \}$
\State $\hat q_1 \gets \text{Quantile}\left(\mathcal{S}_1;\frac{\lceil(n_{\leq \beta}+1)(1-\alpha_1)\rceil}{n_{\leq \beta}}\right)$

% \State Calculate $\hat{q}_1$ as in Equation~\ref{eq:q_hat_1} using  $\mathcal{C}_{\leq \beta}$ and $\alpha_1$
\State $C_{\alpha_1}^{(l)}(x) \gets [\tone(x) - \hat q_1, \infty]$

% Phase 2
\State \texttt{- Phase 2: Find Upper Bound (FPR) With Highest Utility -}
\State $\mathcal{C}_{>\beta}^{(l)} \gets \{(x_i,y_i)\in\mathcal{C}: y_i > \beta \text{ and } \beta \in C_{\alpha_1}^{(l)}(x_i)\}$
\State $n_{>\beta}^{(l)}  \gets |\mathcal{C}_{>\beta}^{(l)}|$

\State $A^* \gets -\infty$
\For{$\alpha_2 \in \{0, 0.01, \dots, 0.99, 1\}$}
    \State $\mathcal{S}_2 \gets \{ \beta - \ttwo(x_i) \mid (x_i, y_i)\in \mathcal{C}_{>\beta}^{(l)} \}$ %\Comment{\fable{To James: the master wrote $\mathcal{C}_{>\beta,p_1}$ here, a symbol defined nowhere; replaced with the Phase-2 calibration set defined five lines above. Please verify.}}
    \State $\hat q_2 \gets \text{Quantile}\left(\mathcal{S}_2;\frac{\lceil(n_{>\beta}^{(l)}+1)(1-\alpha_2)\rceil}{n_{>\beta}^{(l)}}\right)$
    \State $C_{\alpha_2}^{(r)}(x) \gets [-\infty, \ttwo(x) + \hat q_2]$

    \State $\Thresh(x) \gets \begin{cases}
    \text{No Aid}
    & \beta \notin C_{\alpha_1}^{(l)}(x),\cr
    \text{Give Direct}
    & \beta \notin C_{\alpha_2}^{(r)}(x),\cr
    \text{Survey}
    & \text{otherwise}.
    \end{cases}$

    \State $\hat{A}_{\alpha_2} \gets \widehat{\mathcal{U}}(\Thresh;\mathcal C,\mathcal D)$

    \If{$\hat{A}_{\alpha_2} > A^*$}
        \State $A^* \gets \hat{A}_{\alpha_2}$
        \State $\Threshstar \gets \Thresh$
    \EndIf
\EndFor
\State \Return $\Threshstar$
\end{algorithmic}
\end{algorithm}

\subsection{Evaluation design and splits}
\label{sec:ooa_split}

The images used for training the model have a footprint of $6.72 km \times 6.72 km$ centered on the reported neighborhood coordinate. Given the distribution of these neighborhoods across Africa, there are many common spatial areas between them. To ensure that there are no overlapping areas between points in different folds during a 5-fold cross-validation evaluation, we clustered neighborhoods based on their distances. This approach aims to group all neighborhoods within a cluster together into the same fold, ensuring no overlapping areas between different clusters and folds.

To achieve this, we used DBScan clustering with a minimum number of samples set to 1, allowing the detection of individual points. The minimum distance between two points for inclusion in the same cluster is $9.5 km$, which is the diagonal of the input square ($6.72 \text{ km} \times \sqrt{2} \approx 9.5 \text{ km}$).

Using the DBScan results as initial clusters, we identified some clusters with a large number of members. For example, in Egypt, many neighborhoods along the Nile River were clustered together. To break these large chains, some previous literature conducted visual inspections \citep{Yeh2020, pettersson2023}. To increase the reproducibility, we implemented a systematic procedure in which we applied K-means clustering to further subdivide these large clusters into smaller ones. 

Next, we removed points between the clusters to ensure that the minimum distance between points in different clusters was more than 9.5 km. This process was iterative: after performing K-means clustering, we identified points within each cluster that were less than 9.5 km from points in other clusters. We removed the point with the maximum number of close neighbors, then repeated the process until no points in a cluster were less than 9.5 km from any point in another cluster. As a conclusion, this K-means-based procedure aimed to balance the number of points within clusters while minimizing the number of removed points.

Initially, K-means clustering was applied to clusters with more than 500 points to subdivide them into smaller clusters. Following this, a country-based investigation using Shannon entropy was conducted to identify countries with unbalanced number of points per cluster. The Shannon entropy values from DBScan, the first round of K-means on large clusters, and the second round of K-means based on Shannon entropy are shown below. Finally, 67,829 points remained in the dataset.

\begin{table}[!htbp]
\caption{Shannon entropy of clusters per country}
\label{tab:cluster_entropy}
\begin{center}
\resizebox{\columnwidth}{!}{\begin{tabular}{c|c c c c }
\textbf{Country} & DBScan & First Kmeans & Second Kmeans & Shannon based \\
\hline
\textbf{Angola}                    &  7.21  &  7.21  &   7.21   &  7.21  \\
\textbf{Burkina Faso}              &  6.62  &  6.62  &   6.62   &  6.62  \\
\textbf{Benin}                     &  3.00  &  4.33  &   4.33   &  4.33  \\
\textbf{Burundi}                   &  0.01  &  1.39  &   2.57   &  3.51  \\
\textbf{Congo - Kinshasa}          &  8.39  &  8.41  &   8.43   &  8.43  \\
\textbf{Central African Republic}  &  5.03  &  5.03  &   5.03   &  5.03  \\
\textbf{Côte d’Ivoire}             &  7.26  &  7.26  &   7.26   &  7.26  \\
\textbf{Cameroon}                  &  6.56  &  6.56  &   6.56   &  6.56  \\
\textbf{Egypt}                     &  1.02  &  2.93  &   4.80   &  4.80  \\
\textbf{Ethiopia}                  &  8.13  &  8.13  &   8.13   &  8.13  \\
\textbf{Gabon}                     &  5.92  &  5.92  &   5.92   &  5.92  \\
\textbf{Ghana}                     &  5.37  &  6.02  &   6.02   &  6.02  \\
\textbf{Gambia}                    &  2.01  &  2.55  &   2.55   &  2.70  \\
\textbf{Guinea}                    &  6.61  &  6.61  &   6.61   &  6.61  \\
\textbf{Kenya}                     &  4.05  &  5.41  &   6.01   &  6.01  \\
\textbf{Comoros}                   &  1.45  &  1.45  &   1.45   &  2.10  \\
\textbf{Liberia}                   &  4.21  &  5.02  &   5.02   &  5.03  \\
\textbf{Lesotho}                   &  0.06  &  2.35  &   2.35   &  3.32  \\
\textbf{Morocco}                   &  6.96  &  6.96  &   6.96   &  6.96  \\
\textbf{Madagascar}                &  8.07  &  8.07  &   8.07   &  8.07  \\
\textbf{Mali}                      &  7.77  &  7.77  &   7.77   &  7.77  \\
\textbf{Mauritania}                &  6.58  &  6.58  &   6.58   &  6.58  \\
\textbf{Malawi}                    &  1.49  &  3.35  &   4.07   &  5.00  \\
\textbf{Mozambique}                &  7.56  &  7.63  &   7.64   &  7.66  \\
\textbf{Nigeria}                   &  8.37  &  8.37  &   8.37   &  8.37  \\
\textbf{Niger}                     &  7.14  &  7.14  &   7.14   &  7.14  \\
\textbf{Namibia}                   &  6.42  &  6.42  &   6.42   &  6.42  \\
\textbf{Rwanda}                    &  0.00  &  1.77  &   3.66   &  3.66  \\
\textbf{Sierra Leone}              &  2.44  &  3.87  &   3.87   &  4.54  \\
\textbf{Senegal}                   &  3.59  &  4.77  &   4.77   &  5.21  \\
\textbf{Eswatini}                  &  3.15  &  3.15  &   3.15   &  3.15  \\
\textbf{Chad}                      &  8.16  &  8.16  &   8.16   &  8.16  \\
\textbf{Togo}                      &  4.21  &  4.45  &   4.45   &  4.45  \\
\textbf{Tanzania}                  &  8.03  &  8.04  &   8.04   &  8.06  \\
\textbf{Uganda}                    &  5.06  &  5.45  &   6.21   &  6.21  \\
\textbf{South Africa}              &  7.49  &  7.50  &   7.50   &  7.50  \\
\textbf{Zambia}                    &  8.07  &  8.07  &   8.07   &  8.07  \\
\textbf{Zimbabwe}                  &  7.43  &  7.43  &   7.43   &  7.43  \\

\hline
\end{tabular}}
\end{center}
\end{table}

As an example, for Burundi, the initial DBScan algorithm groups almost all samples into one cluster due to their distance (Supplementary Figure~\ref{fig:burundi_clusters}a). By applying multiple steps of KMeans clustering, the area is divided into more clusters, and the points among the clusters are removed. Supplementary Table~\ref{tab:cluster_entropy} shows how this process affects the results.

\begin{figure}[!htbp]
     \centering
     \begin{subfigure}[hbt!]{0.5\textwidth}
         \centering
         \includegraphics[width=\textwidth]{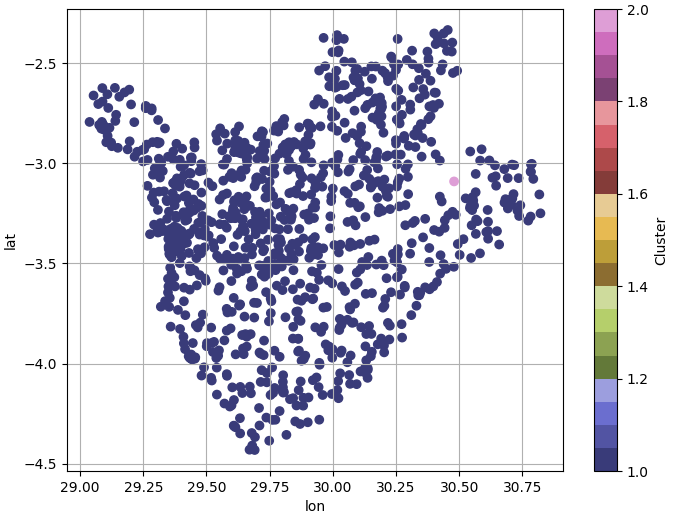}
          \caption{ }
     \end{subfigure}%
     \begin{subfigure}[hbt!]{0.5\textwidth}
         \centering
         \includegraphics[width=\textwidth]{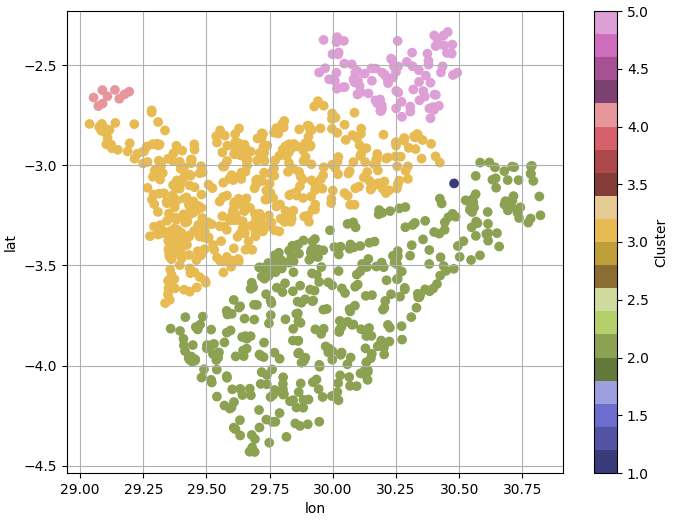}
         \caption{ }
     \end{subfigure}
     \begin{subfigure}[hbt!]{0.5\textwidth}
         \centering
         \includegraphics[width=\textwidth]{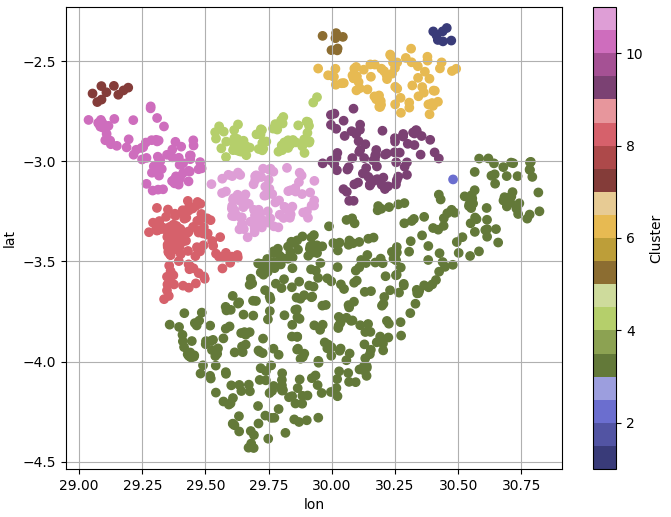}
          \caption{ }
     \end{subfigure}%
     \begin{subfigure}[hbt!]{0.5\textwidth}
         \centering
         \includegraphics[width=\textwidth]{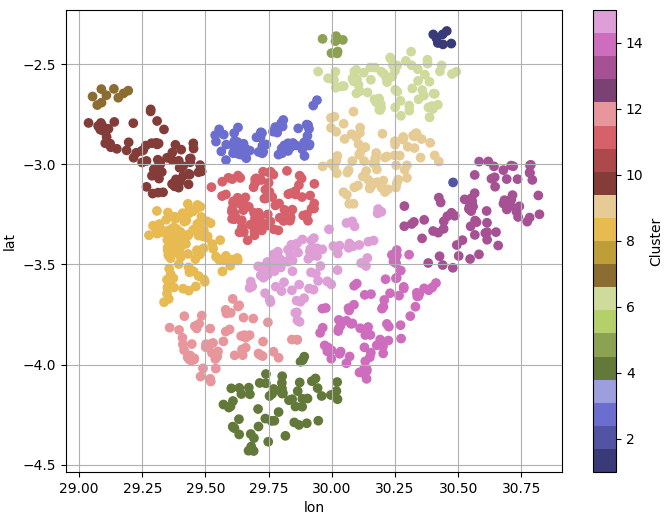}
         \caption{ }
     \end{subfigure}%
     \caption{ Out of area clusters in Burundi. (a) Result from DBScan. (b) Applying the first round of KMeans. (c) Applying the Second round of KMeans. (d) Applying KMeans to countries with low Shannon entropy}
     \label{fig:burundi_clusters}

\end{figure}

To generate the folds, we aimed to balance them according to country clusters, which would naturally lead to balanced folds. We started with the country containing the largest number of points and proceeded to the country with the fewest points. For each country, we began with the largest cluster, assigning it to the fold with the fewest points from that country. This procedure not only balanced the number of points per country in each fold but also ensured an overall balance in the number of points across all folds.  The final number of samples in the five folds are 14,128, 13,726, 13,576, 13,357, and 13,042. Since all points within a cluster were assigned to the same fold, there is some variation in the number of points between folds.

All the results are evaluated with five-fold cross-validation, where the non-training held-out fold is further partitioned into calibration and test subsets for conformal prediction.

\begin{table}[!htbp]
\centering
\caption{Countries assigned to each cross-validation fold during out-of-country training.}
\label{tab:fold_countries}
\begin{tabular}{lp{11cm}}
\toprule
\textbf{Fold} & \textbf{Countries} \\
\midrule
A & Cameroon, Chad, Democratic Republic of the Congo, Egypt, Eswatini, Guinea, Mauritania \\
B & Comoros, Gabon, Namibia, Niger, Nigeria, Rwanda, Uganda, Zambia \\
C & Angola, Côte d'Ivoire, Ethiopia, Kenya, Liberia, Mali, Togo \\
D & Burundi, Gambia, Ghana, Madagascar, Morocco, Sierra Leone, Tanzania, Zimbabwe \\
E & Benin, Burkina Faso, Central African Republic, Lesotho, Malawi, Mozambique, Senegal, South Africa \\
\bottomrule
\end{tabular}
\end{table}

\section{Previous work}
\label{sec:previous_works}

Comparing results across studies remains difficult due to differing datasets, covariates and assumptions.

As a complement, we have trained several baseline architectures on the same dataset and splits, to show that our SQR model's point accuracy is representative of what standard architectures achieve on these data rather than an artifact of model choice (Supplementary Table~\ref{tab:previous_works}).

\begin{table}[!htbp]
\centering
\begin{tabular}{llll}
\textbf{Paper}             & \textbf{Model type} & $R^2$ & \textbf{Covariates} \\ \hline
\citet{pettersson2023}             & CNN + LSTM          & 0.76           & Landsat + NL        \\
\citet{marty2024global} & XGBoost             & 0.75           & EO + other sources  \\
\citet{Wang2024}                   & RF                  & 0.70           & NL                  \\
\citet{zheng2025dynamic}           & SwinV2-T            & 0.69           & Landsat             \\ 
\citet{Yeh2020}                    & CNN                 & 0.67           & Landsat + NL        \\
\citet{sherman2026global}          & MOSAIKS             & 0.67           & MOSAIKS             \\
\citet{Jean2016}                   & CNN                 & 0.56           & Google Maps imagery         \\
\citet{chi_microestimates_2022}  & Gradient Boosting   & 0.56           & EO + other sources  \\
\hline
\hline
\multirow{3}{*}{Other baselines}                       & Small-ViT           & 0.69           & Landsat + NL        \\
                       & ResNet-50           & 0.67           & Landsat + NL        \\
                       & ResNet-18           & 0.65           & Landsat + NL        \\
\hline
\hline
Our EO-SQR model                       & SQR-transformer     & 0.75           & Landsat + NL        \\ 
\hline
\end{tabular}
\caption{\textbf{Previous works} Performance reported in earlier works trained to predict asset wealth from DHS data. Direct comparisons between different works remains difficult, but nevertheless, our proposed EO-SQR model appears competitive with reported figures.}
\label{tab:previous_works}
\end{table}

\section{Map creation}
\label{sec:map_creation}

To ensure privacy and conserve computational resource, we only generate map predictions for raster grid cells with at least 20 inhabitants according to \citet{worldpop_r2025a}. For the 2021 target year, each of the five outer-fold models predicts the 5th, 50th, and 95th conditional IWI quantiles at every retained grid center. The point predictions presented in the left panel of Figure~\ref{fig:iwi_and_ci_maps} display the average of the five median predictions.

In accordance with \citet{gasparin2024merging}, we aim to merge the five prediction intervals by majority voting. Assuming that the five intervals all overlap, we aggregate these intervals by taking the median lower endpoint and the median upper endpoint across folds. The right panel in Figure~\ref{fig:iwi_and_ci_maps} displays the distance between these aggregated endpoints.

These maps are useful for visualization, but the coverage across the raster has not been established. Each conformal correction applies to a new observation exchangeable with one fold's calibration sample. This will not be the case for historical displaced DHS clusters and the 2021 continental grid. We therefore describe the displayed map and its interval widths as descriptive.

The checked-in grid centers are spaced by approximately 0.05875 degrees, about 6.5 km north--south, with east--west distance varying by latitude. This spacing is distinct from the $6.72\times6.72$ km Landsat input footprint. A separate population-enrichment helper sums the WorldPop Global 2015--2030 R2025A constrained 100 m product within 6.27 km squares. These three spatial quantities should not be used interchangeably.

%TC:endignore

\end{document}